%% file: paper.tex
\documentclass{article}

\usepackage{microtype}
\usepackage{graphicx}
\usepackage{booktabs} 
\usepackage{tikz}
\usepackage{subcaption}

\usepackage{hyperref}

\usepackage[accepted]{icml2024}

\usepackage{url}            
\usepackage{booktabs}       
\usepackage{amsfonts,amsmath,amssymb,amsthm}       
\usepackage{xcolor}         
\usepackage{graphicx}
\usepackage{wrapfig}
\usepackage{eqparbox}
\renewcommand{\algorithmiccomment}[1]{\hfill\eqparbox{COMMENT}{\# \textcolor{blue}{#1}}}

\input{defs}

\icmltitlerunning{A Bayesian Approach to Online Planning}

\begin{document}

\twocolumn[
\icmltitle{A Bayesian Approach to Online Planning}



\icmlsetsymbol{equal}{*}

\begin{icmlauthorlist}
\icmlauthor{Nir Greshler}{gm}
\icmlauthor{David Ben Eli}{gm}
\icmlauthor{Carmel Rabinovitz}{gm}
\icmlauthor{Gabi Guetta}{gm}
\icmlauthor{Liran Gispan}{gm}
\icmlauthor{Guy Zohar}{gm}
\icmlauthor{Aviv Tamar}{technion}
\end{icmlauthorlist}

\icmlaffiliation{gm}{General Motors, Advanced Technical Center, Israel}
\icmlaffiliation{technion}{Department of Electrical and Computer Engineering, Technion - Israel Institute of Technology, Haifa, Israel}

\icmlcorrespondingauthor{Nir Greshler}{nir.greshler@gm.com}
\icmlcorrespondingauthor{Aviv Tamar}{avivt@technion.ac.il}

\icmlkeywords{Reinforcement Learning, Planning, Learning to Search}

\vskip 0.3in
]



\printAffiliationsAndNotice{}  

\begin{abstract}
  The combination of Monte Carlo tree search and neural networks has revolutionized online planning. As neural network approximations are often imperfect, we ask whether uncertainty estimates about the network outputs could be used to improve planning. 
  We develop a Bayesian planning approach that facilitates such uncertainty quantification, inspired by classical ideas from the meta-reasoning literature. 
  We propose a Thompson sampling based algorithm for searching the tree of possible actions, for which we prove the first (to our knowledge) finite time Bayesian regret bound, and propose an efficient implementation for a restricted family of posterior distributions. In addition we propose a variant of the Bayes-UCB method applied to trees. 
  Empirically, we demonstrate that on the ProcGen Maze and Leaper environments, when the uncertainty estimates are accurate but the neural network output is inaccurate, our Bayesian approach searches the tree much more effectively. In addition, we investigate whether popular uncertainty estimation methods are accurate enough to yield significant gains in planning.
  Our code is available at: \url{https://github.com/nirgreshler/bayesian-online-planning}.
\end{abstract}

\section{Introduction}

Online planning is fundamental to various decision making problems, ranging from game playing, such as Chess and Go~\cite{silver2018general}, to robotic manipulation and navigation~\cite{finn2017deep,shim2003decentralized}, autonomous driving~\cite{williams2017information,cesari2017scenario}, and more recently, planning with large language models~\cite{zhang2023planning,mankowitz2023faster}. In the standard problem setting, as we consider here, a model of the world is known, its state is fully observed, and an agent must sequentially take actions that yield a high cumulative reward.

For almost all realistic problems, calculating the optimal sequence of actions is intractable, and some approximations must be made. For the past two decades, the dominant approximation approach has been Monte-Carlo Tree Search (MCTS) -- a stochastic traversal of the search tree that balances exploration and exploitation using an upper confidence bound (UCB)~\cite{kocsis2006bandit}. Breakthrough performance in several games was achieved by the seminal AlphaZero, an extension of MCTS with neural network approximations of a value function and a policy, which considerably cut down the search effort~\cite{silver2018general}.

\begin{figure}[t]
    \centering

    \input{tikz/intro_schematic}
    \caption{Example of value estimation errors during search}
    \label{fig:intro_example}
    \vspace{-1em}
\end{figure}
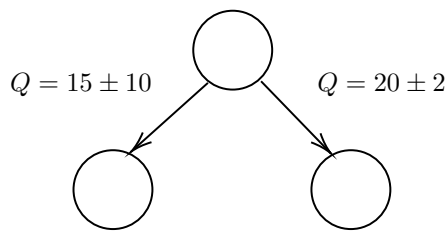
When the online search has a limited budget, as is the case in any real-time application, errors in the neural network approximations can be problematic for MCTS. As an example, consider the situation in Figure \ref{fig:intro_example}, where the search has reached a state with two possible actions, with action-value estimates of $15$ and $20$. 
MCTS will choose the second, higher value action repeatedly, until the low visit count of the unexplored action in the UCB term will dominate, and only then will it explore the first action. However, if we were to know in advance that the uncertainty in the action-value estimates are $\pm 10$ and $\pm 2$, respectively, then we should explore the first action much more frequently, as its low-value estimate may likely be wrong. Unfortunately, MCTS, and its AlphaZero variants are frequentist methods, and do not naturally take such uncertainty information into account.\footnote{We are mainly interested in \textit{epistemic} uncertainty here~\cite{der2009aleatory}. For \textit{aleatoric} uncertainty, frequentist methods such as \citet{audibert2009exploration} are well suited.} 


We advocate here a Bayesian approach to online search. Our premise is that the Bayesian method can naturally exploit uncertainty estimates of the neural network approximations, to yield better performance, especially under modest search budgets and inaccurate neural network predictions. 

Bayesian search algorithms were explored in the planning literature already in the 1990s~\cite{russell1991right,dearden1998bayesian}. For single stage decision making, Bayesian optimization~\cite{shahriari2015taking,frazier2018tutorial,chen2018bayesian} can be seen as a modern incarnation of similar ideas, and meta reinforcement learning is essentially a Bayesian approach to multi-task RL~\cite{zintgraf2019varibad}. The work of \citet{tesauro2010bayesian} pioneered the Bayesian approach to MCTS. However, to our knowledge, during the recent deep learning-fueled revival of online planning~\cite{silver2017mastering,anthony2017thinking,silver2018general,schrittwieser2020mastering}, Bayesian methods have so far been ignored. In this work, we aim to rectify this matter.

We develop both the fundamental and practical aspects of a Bayesian approach to online planning and learning. Our first contribution is a Bayesian formulation of the tree search problem, and a corresponding Thompson sampling based tree search algorithm. We establish a Bayesian regret bound for our algorithm, based on modern analysis techniques~\cite{russo2016information}, which to our knowledge is the first regret analysis of a Bayesian tree search approach. Importantly, our bound shows that when the Shannon entropy of the prior is small (equivalent to high certainty in the neural net approximation), the expected regret is small. Our second contribution is a practical implementation of Thompson sampling tree search, by incorporating efficient methods for sampling from and updating the posterior. Interestingly, our methods bear resemblance to techniques suggested in previous works such as \citet{tesauro2010bayesian}; our formulation establishes them as concrete instances of the Thompson sampling method. In addition, we propose an adaptation of the Bayes-UCB method of \citet{kaufmann2012bayesian} to tree search, which we find to work very well in practice. 
Finally, in the spirit of AlphaZero and Expert Iteration~\cite{anthony2017thinking,silver2018general}, we incorporate deep learning of value functions into our approach using self play. Different from prior work, however, our Bayesian planning algorithms make explicit use of \textit{uncertainty estimates} about the neural network predictions, and we discuss how such could be obtained. 

We evaluate our method on procedurally generated Maze and Leaper environments from the ProcGen benchmark~\cite{cobbe2020leveraging}.
In the setting we investigate, the agent is tested on domains it has not been trained on, and therefore we expect some errors in its neural network approximations.
With access to accurate uncertainty estimates (which can easily be computed for the maze domain), our Bayesian approaches significantly outperform MCTS, validating our main premise. However, with two popular methods for \textit{learning} the epistemic uncertainty, we could not obtain predictions accurate enough to translate to performance gains in planning, suggesting that more research is required to fully harness the potential of the Bayesian paradigm.

\section{Bayesian Online Planning}
\label{sec:bayesian_online_planning}

We consider an agent that sequentially interacts with a dynamic environment in discrete time steps. 
At each time step, the agent observes the current environment state and performs an action. Subsequently, the environment transitions to a new state according to some transition law. 
We assume that the agent has a model of the environment, and at each step can use the model to plan the next course of action. 
For the sake of planning, we assume that the environment model is deterministic. However, we note that since the agent replans at each time step, our solution can also be applied to non-deterministic systems~\cite{yoon2007ff}.
In the following, we focus on the planning problem that needs to be solved at each time step, 
and propose a Bayesian approach for it.

\subsection{Bayesian Tree Search}\label{sec:formulation}


Consider a deterministic decision process $\tree$ with an initial (root) state $s_0$, and a finite action set $A$. Let $s_{n+1} = f(s_n,a_n)$ denote the deterministic dynamics, and let $r(s_n,a_n)\in [-R_{max},R_{max}]$ denote a \textit{deterministic} reward for a state-action pair. We consider decision processes of depth $\hor$, that is, we wish to maximize: 
\begin{equation}\label{eq:decision_process}
\begin{split}
    \max_{a_0,\dots,a_{\hor-1}} & \sum_{n=0}^{\hor-1} r(s_n,a_n), \\
    \textrm{s.t. } & s_{n+1} = f(s_n,a_n), \quad \forall n \in 0,\dots, \hor-1.
\end{split}
\end{equation}
A decision process is equivalent to a tree of depth $\hor$, and henceforth we will refer to it as such. While we do not denote it explicitly, we assume that states at different levels of the tree are distinct (e.g., by having the level of the tree be part of the state). Naively, one can solve \eqref{eq:decision_process} by evaluating all the $A^\hor$ possible $\hor$-length action sequences (e.g., using breadth first search). We will be interested in problems where $A$ and $\hor$ are such that this approach is not tractable.


Intuitively, we would like to focus the search on the more promising parts of the search tree, assuming that we have some prior knowledge about where the optimal solution may lie. In the following, we cast this idea within a formal probabilistic interpretation. Our main insight, inspired by the meta reasoning literature~\cite{russell1991right}, is that each edge expansion during the tree search is equivalent to querying the rewards for the actions of a state. Thus, in terms of the number of computations required, finding the optimal action sequence when rewards are known is equivalent to identifying the rewards of the optimal action sequence when rewards are not known. The latter, however, is much more convenient to interpret probabilistically.

\begin{figure}
    \centering
    \resizebox{0.9\linewidth}{!}{
    \input{tikz/formulation_schematic.tex}}
    \caption{Illustration of the formulation in Section \ref{sec:formulation}. A tree $\tree$ of depth $H=3$ is shown. Let the action $\{L,R\}$ correspond to the left and right transitions, respectively. Assume that the optimal branch is $(S_0,R)\to(S_1,L)\to(S_2,R)$. Then, $\leaf^* = (S_2,R)$. At time $t=4$, the state-action pairs that have already been explored are marked in solid line, and the next state-action to be explored is $\leaf_t = (S_1,R)$. The set $\leaves_t$ is marked in purple. Note that $\leaf^*$ is indicative of the optimal branch, and also of $\leaf^*_t$, and of the optimal action at the root, $A^*$.}
    \label{fig:formulation_schematic}
\end{figure}
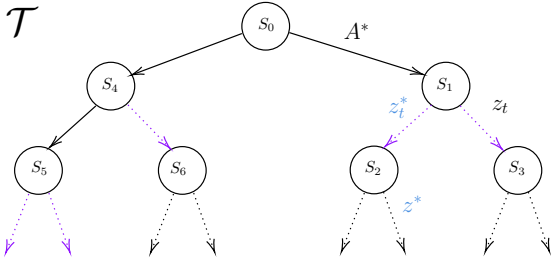

A tree is completely characterized by the rewards $r(s, a)$. We assume a prior distribution over the rewards, which induces a distribution over the trees, which we denote as $P(\tree)$.\footnote{For simplicity, we assume that $P(\tree)$ is a discrete distribution, but our derivations extend to continuous distributions by replacing sums with integrals.}
For a tree $\tree$, let $Q_n(s,a;\tree)$ denote its corresponding state-action value function, defined as follows:
\begin{equation*}
\begin{split}
    Q_n(s,a) = & \max_{a_{n+1},\dots,a_{\hor-1}} \sum_{\tau=n+1}^{\hor-1} r(s_{\tau},a_{\tau}),  \\
    & \textrm{s.t. } \quad s_{n} = s, \quad a_n = a, \\
    & \qquad s_{\tau+1} = f(s_\tau,a_\tau), \quad \forall \tau \in n,\dots \hor-1.
\end{split}
\end{equation*}
The prior distribution over value functions is given by
$
    P\left(Q_n(s,a) = \alpha\right) = \sum_{\tree} P(\tree) \mathbf{1}\left\{Q_n(s,a;\tree) = \alpha\right\}.
$

We consider a sequential and stochastic discovery of the tree that takes place over $T$ iterations ($T$ will be termed the \textit{search budget}),
where at each iteration $t \in \{1,\dots,T\}$, a reward for a particular leaf state-action pair $\leaf_t = (s_t,a_t)$ is revealed; an illustration is provided in Figure \ref{fig:formulation_schematic}. 
Recall that this sequential discovery relates to the planning that happens \textit{at each} time step of the online planning scheme described above, and should yield the optimal action to take at the root node, which is the current state of the environment.
Let $\leaves$ denote the set of all state-action pairs in the tree, and let $\leaves_t$ denote the set of \textit{leaf} state-action pairs in the tree that has been discovered up to iteration $t$. That is, we have that for all $t$, $\leaf_t \in \leaves_t$, and $\leaves_t \subset \leaves$.
Let $\obs_t = \left\{(s_t,a_t),r(s_t,a_t)\right\} \equiv \left\{\leaf_t,r(\leaf_t)\right\}$ denote the observation at time $t$, and let $\hist_t = \left\{\obs_0,\dots, \obs_{t-1} \right\}$, where $\obs_0 = \{\}$, denote the history ($\sigma$-algebra) of the discovery process at time $t$.\footnote{While we assume here that rewards are deterministic, our derivation extends to stochastic rewards, where each observation reveals an i.i.d.~sample from the reward.} We will consider the posterior distributions $P(\tree|\hist_t)$ and $P(Q_n(s,a)|\hist_t)$, which are well defined.

Given $\tree$, an optimal action at the root is well defined:
\begin{equation}\label{eq:opt_action}
    A^* \in \argmax_{a \in A} Q_0(s_0,a;\tree).
\end{equation}

For simplicity, we will assume that for any possible tree, the optimal action at the root is unique. For a leaf state-action pair $\leaf_t$ that is explored at iteration $t$, let $\roota(\leaf_t)\in A$ denote the action at the root that leads to $\leaf_t$. Also, let $\leaf^*_t$ denote the leaf available at time $t$ that is on the optimal branch (if the branch is unique, then the leaf is unique), and let $\leaf^*$ denote the leaf of the complete tree on the optimal branch. Note that $\roota(\leaf^*_t) = \roota(\leaf^*) = A^*$.

We define the $T$ period \textit{regret} of the sequence of state-action pairs $\leaf_1,\dots,\leaf_T$ as the random variable,
$
    \textrm{Regret}(T) = \sum_{t=1}^T \left[Q_0(s_0,A^*) - Q_0(s_0,\roota(\leaf_t))\right].
$
Note that minimizing regret is equivalent to minimizing the error due to a suboptimal action \textit{at the root}, which, as explained above, is what ultimately matters for the online planning scheme. We shall study the expected regret, a.k.a. Bayesian regret,
\begin{equation}\label{eq:Bayes_regret}
    \mathbb{E}\left[\textrm{Regret}(T)\right] \!=\! \mathbb{E}\!\left[\sum_{t=1}^T \left[Q_0(s_0,A^*) \!-\! Q_0(s_0,\roota(\leaf_t))\right]\right]\!,
\end{equation}
where the expectation is taken over the randomness in the action selection, and over the prior distribution over $\tree$.

We shall now propose a general tree search algorithm, and then analyse its Bayesian regret. Our tree search algorithm is based on the Thompson Sampling idea~\cite{thompson1933likelihood}, and selects leaves according to their posterior probability of being on the optimal branch:
\begin{equation}\label{eq:TS_action_prob}
    P(\leaf_t = \leaf | \hist_t) = P(\leaf_t^* = \leaf | \hist_t).
\end{equation}
For a random variable $X$, let $\ent(X)$ denote its Shannon entropy. The next theorem bounds the regret of our algorithm. 
\begin{theorem}\label{thm:TSTS_regret}
The regret of the leaf selection rule defined in Eq.~\eqref{eq:TS_action_prob} satisfies:
$
    \mathbb{E}\left[\textrm{Regret}(T)\right] \leq \hor R_{max}\sqrt{\frac{1}{2}|\leaves|\ent(\leaf^*)T}.
$
\end{theorem}
Note the dependence on $\ent(\leaf^*)$ -- if the prior over $\leaf^*$ has low entropy, i.e., it is an \textit{informative} prior, the problem becomes easier. This is a property of the Bayesian analysis. 
To further demonstrate this point, the following example compares the performance of Thompson Sampling based tree search to a search of the tree that is agnostic to the prior (e.g., breadth first search or depth first search).
\begin{example}  
Consider a distribution of trees where each tree has exactly one state-action pair with non-zero reward, equal to $1$. Without loss of generality, at each time step $t$ before the non-zero reward is found, the agnostic search has a probability of $1/(|\leaves|-t+1)$ to discover the reward, and the regret is $1$. After the reward has been found, the regret is $0$. This leads to expected regret $\mathbb{E}\left[\textrm{Regret}(T)\right]=T(1 - \frac{T}{2 |\leaves|})$. In particular, for $T=|\leaves|$, we have $\mathbb{E}\left[\textrm{Regret}(T=|\leaves|)\right]=\frac{|\leaves|}{2}$. The bound in Theorem \ref{thm:TSTS_regret}, on the other hand, gives $\mathbb{E}\left[\textrm{Regret}(T=|\leaves|)\right]=|\leaves|\sqrt{\frac{\ent(\leaf^*)}{2}}$, where we ignore the $\hor R_{max}$ terms as each trajectory can have at most reward $1$. When the prior is uniform, $\ent(\leaf^*) = \log |\leaves|$, giving a worse regret than for the agnostic search. However, if the prior is such that $\ent(\leaf^*)<0.5$, Thompson sampling will exploit this structure to obtain a lower regret.
\end{example}

The proof of Theorem \ref{thm:TSTS_regret} builds on the information theoretic analysis of Thompson Sampling for multi armed bandits (MABs) by \citet{russo2016information}, and is detailed in Section \ref{sec:proof}. We adapt their analysis to our case by defining a MAB problem that has similar regret as \eqref{eq:Bayes_regret}, but with non i.i.d.~rewards, and exploit the fact that their analysis does not depend on the rewards being i.i.d.\footnote{While \citet{russo2016information} do consider a case of i.i.d. rewards, their analysis does not actually use this property at all. We note that this fact was previously observed in other studies such as \citet{bubeck2016multi}. }

\subsection{Practical Thompson Sampling Tree Search}

\begin{algorithm}[h]
\caption{\methodfull}\label{alg:tsts}
\begin{algorithmic}
    \REQUIRE $P_{\textrm{query}}(Q(s,a))$, $r_{\textrm{query}}(s,a)$.
    \STATE Init known set $S_{\textrm{known}} := \{s_0\}$
    \STATE Init value $P(Q(s_0,a)) := P_{\textrm{query}}(Q(s_0,a)) \quad \forall a\in A$. 
    \FOR{t = 1,2,\dots,T} 
    \STATE \COMMENT{Forward sampling}
    \STATE Set $s:=s_0$, $s':=\emptyset$
    \WHILE{True}
        \STATE For each $a\in A$, sample $Q(s,a)\sim P(Q(s,a))$ \hfill \textcolor{blue}{[*]}
        \STATE Set $s' := f(s, \argmax_{a\in A} Q(s,a))$\hfill \textcolor{blue}{[**]}
        \IF{$s' \notin S_{\textrm{known}}$} 
            \STATE Set $ S_{\textrm{known}} := S_{\textrm{known}} \cup \{s'\} $
            \STATE Break
        \ENDIF
        \STATE Set $s := s'$
        \ENDWHILE
    \STATE Set value distribution for leaf $P(Q(s',a')) = P_{\textrm{query}}(Q(s',a')) \quad \forall a'\in A$.
    \STATE {}\COMMENT{Max-Backup}
    \WHILE{True}
        \STATE Set $(s,a) := f^{-1}(s')$ 
        \STATE Set $P(Q(s,a)) := P(r(s,a) + \max_{a'\in A} Q(s',a') )$
        \IF{$s = s_0$}
            \STATE Break
        \ENDIF
        \STATE Set $s' := s$
    \ENDWHILE
    \ENDFOR
\end{algorithmic}
\end{algorithm}

The previous section established that the Thompson sampling strategy is a sound exploration method for tree search. Practically, however, each iteration of Thompson sampling involves sampling a leaf in the tree from the posterior $P(\leaf_t^* | \hist_t)$. In general, computing the posterior probability and sampling from it can be computationally demanding. In the following, we propose an efficient method for the special family of independent $Q$ value posteriors. 

We shall use the following notation. For independent random variables $X_1,\dots,X_n$ with distributions $P(X_1),\dots,P(X_n)$ we denote by $P\left(\max_{i\in 1,\dots,n} \left\{X_i\right\}\right)$ the distribution of their maximum order statistic. For a scalar $b$, we denote by $P(X_1 + b)$ the distribution of the random variable $X_1 + b$.

Given $\tree$, Bellman's optimality equation states that, 
$
     Q_n(s,a; \tree) = r(s,a) + \max_{a'} \left\{ Q_{n+1}(f(s,a),a'; \tree) \right\}.
$

Consider that at iteration $t-1$ of our Thompson sampling algorithm, action $a_{t-1}$ was chosen at state $s_{t-1}$, and $r(s_{t-1},a_{t-1})$ was revealed. Assume that the posteriors for the next state and action values $P(Q_{n+1}(f(s_{t-1},a_{t-1}),a')|\hist_t)$ are independent (with respect to the different actions). Then, we have that the posterior for $Q_n(s_{t-1},a_{t-1})$ is given by,\footnote{Eq. \eqref{eq:posterior_update} also holds for dependent posteriors. For ease of exposition, however, we focus on the independent case.} 
\begin{equation}\label{eq:posterior_update}
\vspace{-0.5em}
\begin{split}
     &P(Q_n(s_{t-1},a_{t-1})|\hist_t) = \\
     &P\!\left(\left.r(s_{t-1},a_{t-1}) \!+\! \max_{a'} \left\{ Q_{n+1}(f(s_{t-1},a_{t-1}),a') \right\}\right|\hist_t\!\right).
\end{split}
\raisetag{3.5em}
\end{equation}
If we further assume that the posterior for branches that do not involve $(s_{t-1},a_{t-1})$ does not change, we can apply the rule in \eqref{eq:posterior_update} recursively to update all the posteriors in the tree. 
We refer to this update as the \textit{max-backup} method.

After updating the posterior $Q$ values, sampling from $P(\leaf_t^* | \hist_t)$ can be done by noting that for any state $s$, we have that $P\left((s,a) \in \textrm{optimal branch}\right) = P\left(Q(s,a) > Q(s,\tilde{a}) \quad \forall \tilde{a}\neq a\right)$. Therefore, we can sample from the optimal branch distribution by sequentially sampling $Q$ values, and choosing the optimal action w.r.t. the sampled $Q$. We term this the \textit{forward sampling} method.

The \methodfull\ method (\methodabbrv) in Algorithm \ref{alg:tsts} combines forward sampling and max-backup into a complete tree search routine. Figure \ref{fig:TSTS_schematic} further illustrates the different steps in the algorithm. 
Our algorithm requires that when we explore a leaf $(s,a)$, we can directly query its reward and next state-action value posteriors, denoted $r_{\textrm{query}}(s,a)$ and $P_{\textrm{query}}(Q(s,a))$, respectively.\footnote{We omit the conditioning of the posteriors $r_{\textrm{query}}$ and $P_{\textrm{query}}(Q(s,a))$ on history to ease notation.} In the sequel, we will realize the posterior query using various learned neural network approximations. 
We mention that the max-backup rule was proposed by \citet{tesauro2010bayesian}, but without a formal derivation. In our formulation, max-backup, when combined with forward sampling, and under the independent posterior distribution assumption, emerges as a natural implementation of the \methodabbrv\ method. We establish this formally in the following proposition.

\begin{proposition}\label{prop:TS_algorithm}
    Assume that at each iteration $t$, the posterior of values for state-actions on the leaves are independent, i.e., $P(Q(s,a): (s,a)\in \leaves_t|\hist_t)=\prod_{(s,a)\in \leaves_t}P(Q(s,a)|\hist_t)$. Then Algorithm \ref{alg:tsts} samples leaves from the TS distribution, $P(\leaf_t^* | \hist_t)$.
\end{proposition}

Let us discuss the validity of the independent posteriors assumption. Indeed, it is easy to imagine problems where inferring a reward in a particular state action is informative about the value of different actions in the same state, or even about the value of other states in the tree. Unfortunately, designing an efficient posterior update and sampling method for this case is non-trivial, and we leave it as an open problem for future work. Empirically, we have found that even under the independent posteriors assumption, our Bayesian approach can yield significant performance improvements.


\subsection{Improved Exploration via Bayes-UCB}

The \methodabbrv\ method resolves the exploration-exploitation tradeoff through posterior sampling. In practice, other exploration methods may perform better, as posterior sampling is not necessarily Bayes-optimal. In this section we propose a different Bayesian exploration strategy that we found to empirically perform very well. 

We propose \methodBfull\ (\methodBabbrv) -- an exploration method based on the idea of optimism in the face of uncertainty, inspired by the Bayes-UCB algorithm of \citet{kaufmann2012bayesian}. The idea is to choose actions at each state proportional to the \textit{quantiles} of  posterior state-action values. 

For a random variable $X$ with distribution $P(X)$, let $\quantile(\qalpha, P)$ be its $\qalpha$-quantile, such that $P(X \leq \quantile(\qalpha, P)) = \qalpha$. 
In \methodBabbrv, we replace the action selection rule in \methodabbrv's Forward Sampling method (lines marked by \textcolor{blue}{[*]} and \textcolor{blue}{[**]} in Algorithm \ref{alg:tsts}) with the following:
\begin{equation}
\label{eq:action_selection_percentile}
\vspace{-0.5em}
    \begin{split}
        a^* &:= \argmax_a \quantile(\qalpha(s), P(Q(s,a))), \\
        s' &:= f(s, a^*),
    \end{split}
\vspace{-0.5em}
\end{equation}
where the quantile level is given by 
$\qalpha(s) = 1 - (1 - \alpha_0) \cdot e^{-\frac{N(s) - 1}{\beta}}$,
where $N(s)$ is the number of visits to state $s$, and the initial quantile $\alpha_0$ and rate coefficient $\beta$
are tunable hyper-parameters. Full pseudo-code is provided in Appendix \ref{sec:bts_pseudo}. 


The Bayes-UCB algorithm of \citet{kaufmann2012bayesian} applied a  selection rule $\qalpha(s) = 1 - \frac{\beta}{N(s)}$ in the MAB setting, where $P(Q(s,a))$ is replaced with the posterior reward probability for each arm, and $N(s)$ is replaced with the iteration number. 
Under the assumptions of Proposition \ref{prop:TS_algorithm}, the Max-Backup method leads to the correct state-action value posterior in each state, and hence \methodBabbrv\ applies Bayes-UCB with a different quantile schedule to tree search by applying it to each state, similarly to the way UCB is adapted to tree search in UCT~\cite{kocsis2006bandit}. The 
Bayes-UCT2 rule of \citet{tesauro2010bayesian} selects an action that maximizes $\mathbb{E}[P(Q(s,a))] + \sqrt{2 \ln N(s) \mathrm{Var}[P(Q(s,a))]}$, which for a Gaussian posterior is equivalent to a quantile schedule $0.5+0.5 \erf( \sqrt{\ln N(s)})$.
Intuitively, in all three schedules, as a node is visited more often, the action selection is more optimistic (higher quantile), exploring actions that have some chance of turning out to be better than the action that currently yields the highest expected return. This intuition is shown in \citet{kaufmann2012bayesian} to yield asymptotically optimal Bayesian regret bounds for the MAB problem with binary rewards. In our experiments, we found the \methodBabbrv\ schedule to outperform both Bayes-UCB and Bayes-UCT2. Adapting the analysis of \citet{kaufmann2012bayesian} to the tree search setting is not trivial, and left to future work. 

\subsection{Action Commitment in Online Planning}\label{ssec:commitment}

To connect our Bayesian tree search methods to the online planning scheme, note that \methodabbrv\ and \methodBabbrv\ return the posterior state-action value distribution at the root state, and also the tree discovered during search, both of which can be used by the online planning scheme to select an action to perform in the environment. We shall term this action selection step as \textit{action commitment}, different from the action selection during tree search. 

We found that using the posterior state-action value distribution for action commitment often yielded unfavorable results, as the max-backup tends to inflate the value of states down the tree with high uncertainty. 
An alternative, which we shall term \textit{MCTS action commitment}, is to select an action that corresponds to the branch with the highest expected return (sum of rewards $r_{\textrm{query}}(s,a)$ and $\mathbb{E}\left[P_{\textrm{query}}(Q(s,a))\right]$ at the leaf). We note that due to the deterministic dynamics and reward, this strategy is equivalent to a standard MCTS algorithm that commits to the action with the highest backed-up value at the root. An even safer strategy is to select an action that corresponds to the branch with the highest $\alpha$-quantile of the return. A different method is to choose actions stochastically, according to a SoftMax over the backed-up value at the root; this method is helpful when the agent can get `stuck' by committing to a wrong action over and over again. We explore these commitment strategies in our experiments.

\begin{figure*}[t]
    \begin{subfigure}[b]{0.32\textwidth}
    \centering
    \resizebox{\linewidth}{!}{
        \input{tikz/tree_search_0}
    }
    \caption{Forward Sampling}
    \end{subfigure}
    \begin{subfigure}[b]{0.32\textwidth}
    \centering
    \resizebox{\linewidth}{!}{
        \input{tikz/tree_search_1}
    }
    \caption{Forward Sampling}
    \end{subfigure}
    \begin{subfigure}[b]{0.32\textwidth}
    \centering
    \resizebox{\linewidth}{!}{
        \input{tikz/tree_search_2}
    }
    \caption{Max-Backup}
    \end{subfigure}
    \caption{\methodabbrv\ Algorithm Schematic. Plots (a) and (b) show two successive iterations of forward sampling, where states in $S_{\textrm{known}}$ are marked in gray. Subsequently, in Plot (c), state $s_2$ is added to $S_{\textrm{known}}$, and the max-backup routine is performed to update the posteriors.}
    \label{fig:TSTS_schematic}
\end{figure*}
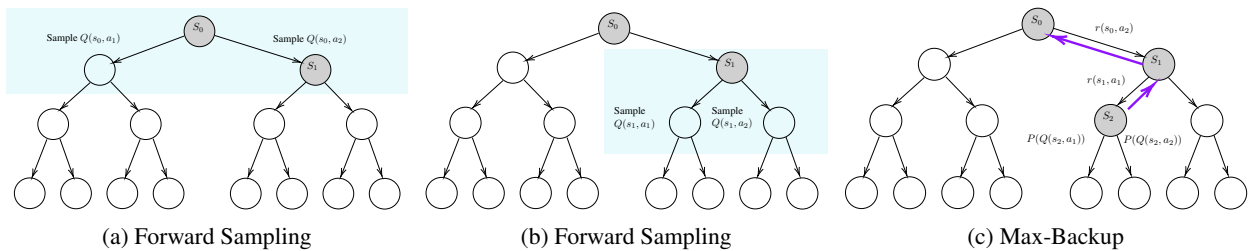

\subsection{Learning in Bayesian Tree Search}
\label{ssec:learning_bayesian_tree_search}

After describing the fundamentals of Bayesian tree search, we are finally ready to combine \methodabbrv/\methodBabbrv\ with deep learning of the state-action value function distribution.

The basic component in our method is a neural network prediction for the posterior value distribution $P_{\textrm{query}}(Q(s,a))=\mathcal{N}(\mu_\theta(s,a), \sigma^2_\theta(s,a))$, where $\mu_\theta(s,a)$ is the output of a neural network with parameters $\theta$, and $\sigma_\theta(s,a)$ is an estimate of the uncertainty in $\mu_\theta(s,a)$. The literature on estimating uncertainty in neural network predictions is extensive~\citep{gawlikowski2021survey}; here, we focus on two simple methods: maximum likelihood estimation (MLE) and ensembles. In the MLE method, the neural network has an additional head for $\log \sigma_\theta(s,a)$~\citep{kendall2018multi}, and the loss function for a data sample $\widehat{Q}(s,a)$ is the negative log likelihood:
$
    \loss = \frac{1}{2}\log\left(\sigma_\theta(s,a)^2\right) + \frac{\left(\mu_\theta(s,a) - \widehat{Q}(s,a)\right)^2}{2 \sigma_\theta(s,a)^2}.
$

In the ensemble method, we have an ensemble of $\numensemble$ neural networks for $\mu(s,a)$, $\mu_{\theta_1}(s,a),\dots,\mu_{\theta_\numensemble}(s,a)$, each trained using the mean-squared error loss $\loss = \sum_{i=1}^{\numensemble}\left(\mu_{\theta_i}(s,a) - \widehat{Q}(s,a)\right)^2$, but with different (random) initial weights. The output of the ensemble is the average, 
$
    \mu_\theta(s,a) = \frac{1}{\numensemble}\sum_{i=1}^{\numensemble}\mu_{\theta_i}(s,a),
$
and the uncertainty estimate is the empirical standard deviation,
$
    \sigma^2_\theta(s,a) = \frac{1}{\numensemble-1}\sum_{i=1}^{\numensemble}\left(\mu_{\theta_i}(s,a) - \mu_\theta(s,a)\right)^2.
$

Our method proceeds in rounds similar to AlphaZero and Expert Iteration~\cite{anthony2017thinking,silver2018general}, where in each round the initial state is reset to $s_0$, and online planning with current neural network parameters $\theta$ is performed for $k$ time steps (or until a terminal state is reached). The learning targets for each root state and actions visited in the online trajectory are the expected posterior $Q$ values at the root per each search.

\section{Related Work}\label{sec:related}
Meta-reasoning is the study of allocating computational resources in artificial intelligence~\cite{russell1991principles,griffiths2019doing}. Selecting which actions to explore during search is known as the \textit{metalevel} decision problem~\cite{hay2012selecting,russell1991right}, and is related to Howard's value of information~\cite{howard1966information} and the Bayesian ranking and selection problem~\cite{frazier2010paradoxes}. In its Bayesian formulation, the optimal sequence of actions is well defined and its computation is equivalent to solving a partially observed MDP~\cite{hay2012selecting}, thus finding it is generally intractable. Approximate solutions include Thompson sampling~\cite{thompson1933likelihood} and the knowledge gradient~\cite{ryzhov2012knowledge}, which is related to the value of perfect information (VPI) heuristic ~\cite{baum1997bayesian,dearden1998bayesian,russell1991right}, and the expected improvement heuristic in Bayesian optimization~\cite{frazier2018tutorial}. 

Several studies applied a Bayesian metalevel decision making approach in MCTS.
\citet{mern2021bayesian} use Gaussian processes for MCTS with continuous actions, and apply the expected improvement heuristic for selecting actions.
\citet{bai2013bayesian} replace the UCB selection rule in each node of the MCTS search tree with Thompson sampling, assuming that $Q$ values are distributed as a mixture-of-Gaussians, and  \citet{bai2014thompson,bai2018posterior} further extends this approach to planning in partially observable problems. 
\citet{tolpin2012mcts,hay2012selecting} replace MCTS's action selection at the root node by approximations to the value of perfect information, and with a UCB update suited for the simple regret. \citet{tesauro2010bayesian} propose a backup of the maximum order statistic, similarly to our max-backup, and used it within an ad hoc UCB-style selection rule. 
Closely related to our work, the recent study by \citet{dam2023monte} models the value distributions along the search tree using Gaussians, and uses the power mean to backup values and their uncertainties, and propose both a UCT and Thompson sampling strategies for action selection.
\citet{bai2013bayesian} and \citet{tesauro2010bayesian} claim asymptotic convergence of their methods to the optimal action in the limit $T \to \infty$, and \citet{dam2023monte} also provide an \textit{asymptotic} polynomial convergence rate. 
Different from the works above, we effectively apply Thompson sampling to the \textit{branches} in the search tree, which allows us to obtain the first \textit{finite-sample regret} guarantees for Thompson sampling tree search. In addition, we investigate the \textit{learning} setting, by connecting the Bayesian posterior to neural network uncertainty estimates. 

\citet{lan2021learning} estimate the neural network uncertainty, and use it to stop the MCTS search when it is certain, reducing average computation time. In contrast, we exploit the uncertainty to design a different search procedure, which is orthogonal to early stopping. \citet{danihelka2021policy} improve Alpha Zero's search by replacing UCB at the root, which minimizes \textit{cumulative} regret, with sequential halving~\cite{karnin2013almost} -- a frequentist algorithm for minimizing the \textit{simple} regret. 
We compare our approach with different action selection and backup procedures based on \citet{danihelka2021policy}, \citet{bai2013bayesian}, and \citet{dam2023monte} in our experiments. 

\citet{jin2015convolutional} is, to our knowledge, the only previous work that used neural networks with a Thompson sampling-based MCTS algorithm. In that work, a policy network was employed for selecting actions during rollouts. In our work, following the successful AlphaZero methodology~\citep{silver2018general}, we use neural networks for value estimates.

\section{Experiments}\label{sec:experiments}
We aim to demonstrate the potential of using uncertainty estimates in online planning. 
However, as the online planning and learning procedure involves several components, our consideration was to design experiments where the contribution of individual components could be clearly teased out and investigated. 
To this end, we focus on two tasks from the ProcGen suite of procedurally generated game environments~\cite{cobbe2020leveraging}, Maze and Leaper. We report an extensive investigation on Maze in the main text, and present similar results on Leaper in the supplementary material Section \ref{sec:leaper_results}.
Designed as a benchmark for zero-shot generalization in deep RL, ProcGen presents a challenge in dealing with epistemic uncertainty, which we hope to mitigate using our Bayesian approach. 
In addition, the maze and leaper domains allow us to calculate ground-truth values for the $Q$ functions and consequently, also for the uncertainty in the neural-network approximation. We begin by describing our experimental setup; comprehensive technical details are provided in Appendix \ref{sec:implementation_details}.

\textbf{Online planning in ProcGen:} ProcGen games are deterministic, with a finite action space. The game state is not directly accessible, but the agent observes a rendering of the game state as an image. The simulator state can be saved, and reset to a saved state, allowing us to implement a model-based planning scheme without access to the true state transitions, by querying the simulator for the image that would be observed upon taking an action. At each time step $i = 1,\dots,k$, the agent is in state $s_i$ and allowed a search budget of $T$ tree search iterations, after which it must commit to an action $a_i$, and the game proceeds to the next time step. We evaluate the agent by whether it reached a rewarding terminal state or not; evaluation by the accumulated reward in the environment gave similar results.

\textbf{Learning in ProcGen:} ProcGen procedurally generates game levels, and we let $\level$ denote a specific level instance (in practice, the random seed used to generate this level). We consider a set of $\trainlevels$ training levels and disjoint set of $\testlevels$ test levels. We train our neural networks on the training levels, and evaluate their performance on the test levels. Previous work~\cite{cobbe2020leveraging} has already established that in the Maze game, for a moderate $\trainlevels$ there is a significant generalization gap, indicating high relevance for epistemic uncertainty. We emphasize that our goal is not to reduce this uncertainty, but only to mitigate its effect on planning. Therefore, we adopt the Impala neural network architecture that was used in previous studies~\cite{espeholt2018impala}, and a moderate $\trainlevels = 150$. 
Further details regarding the training of the neural network are given in Appendix \ref{ssec:training_params}.

\textbf{Evaluation:} Different planning algorithms can be evaluated using the same neural network $\mu_\theta(s,a)$. We differentiate between the planner used for collecting the data for learning, termed the \textit{annotator}, and the planner used for evaluation. In our experiments, we evaluate different planners on a network trained with a single annotator, allowing us to compare different planners on \textit{the same} neural network. 
In addition, we note that some planners are inherently stochastic (e.g., \methodabbrv), while some are deterministic (e.g., UCT). In domains such as mazes, where repeatedly choosing a wrong action would get the agent stuck, stochasticity can be an advantage. To fairly compare planners in such domains, we fix the random seed of a stochastic planner $A$ to be the same at all time steps, and denote such a planner as $A_{\textrm{det}}$. Finally, we specify the action commitment strategies we use for each experiment in the experiment description.

\textbf{Ground Truth Values and Uncertainties:}
In a Maze task, it is straightforward to calculate a ground truth value of $Q(s,a)$, denoted $Q^{\text{GT}}(s,a)$ using algorithms for shortest paths on graphs~\cite{even2011graph}. While our agents do not have direct access to the graph that underlies the observed image, we can use the ground truth value to obtain a ground truth estimate of the neural network uncertainty,
\begin{equation}\label{eq:gt_uncertainty}
    \sigma^{\text{GT}}_\theta(s,a) = |\mu_{\theta}(s,a) - Q^{\text{GT}}(s,a)|.
\end{equation}
We emphasize that in any realistic problem, $Q^{\text{GT}}(s,a)$ and $\sigma^{\text{GT}}_\theta(s,a)$ would not be available at test time, and we use them here only to demonstrate the potential of our algorithms when uncertainty estimation is perfect. 

\textbf{Baselines and Ablations:} We compare \methodabbrv\ and \methodBabbrv\ with the following baselines. \nmctsalg: a neural MCTS algorithm based on \citet{anthony2017thinking,silver2017mastering}. For a fair comparison with our methods, we train a single Q-network, and use it both for value estimation and a SoftMax policy for searching the tree using P-UCT\footnote{The alpha-zero implementation in \cite{silver2017mastering}, for example, had different networks for the policy and the value functions, making it harder to compare with \methodabbrv.}. The SoftMax temperature was set to $2.0$ using a hyper-parameter search.
\gumbelalg: a variant of \nmctsalg\ inspired by \citet{danihelka2021policy}, where exploration at the root is done using sequential halving instead of P-UCT. Note that \nmctsalg\ is deterministic while \gumbelalg\ is stochastic. \buctalg\ and \bucbalg: the \methodBabbrv\ algorithm, but with the Bayes-UCT2 \citep{tesauro2010bayesian} and the Bayes-UCB \citep{kaufmann2012bayesian} action selection rules, respectively. In all algorithms, best hyper-parameters were searched for; we report on the sensitivity to hyper-parameters in Appendix \ref{sec:senitivity}.
In addition, in Appendix \ref{ssec:comparison_wmcts_dng} we compare our algorithm with the methods of \citet{dam2023monte} and \citet{bai2013bayesian}.

\vspace{-0.5em}
\subsection{Results}\label{ssec:experiments_maze}

\begin{figure*}[ht]
    \begin{subfigure}[b]{0.33\textwidth}
    \centering
    \resizebox{\linewidth}{!}{
        \includegraphics[width=\columnwidth]{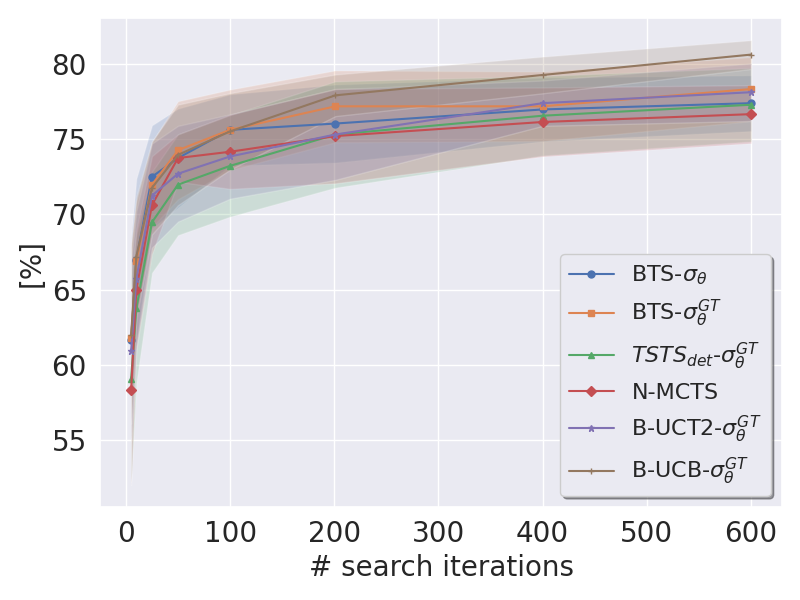}
    }
    \caption{Training domains}
    \label{fig:planner_comparison_train}
    \end{subfigure}
    \begin{subfigure}[b]{0.33\textwidth}
    \centering
    \resizebox{\linewidth}{!}{
        \includegraphics[width=\columnwidth]{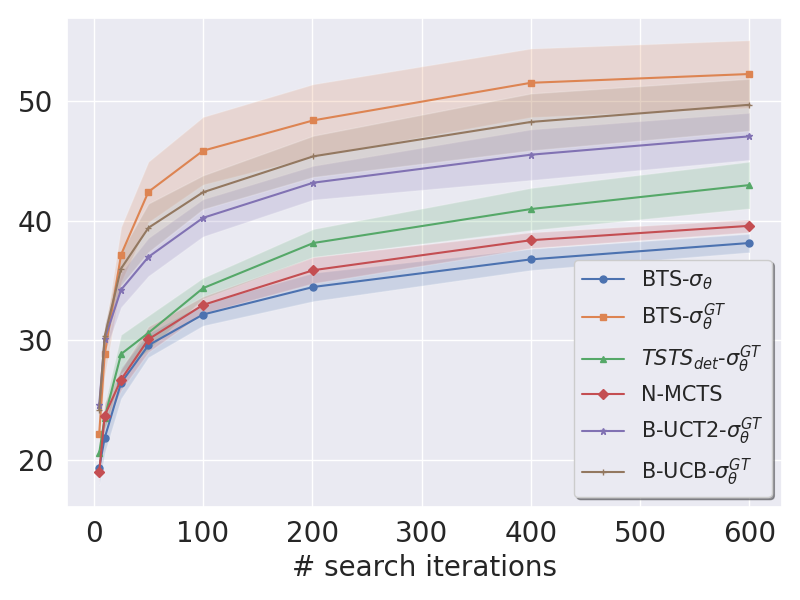}
    }
    \caption{Test domains}
    \end{subfigure}
    \begin{subfigure}[b]{0.33\textwidth}
    \centering
    \resizebox{\linewidth}{!}{
        \includegraphics[width=\columnwidth]{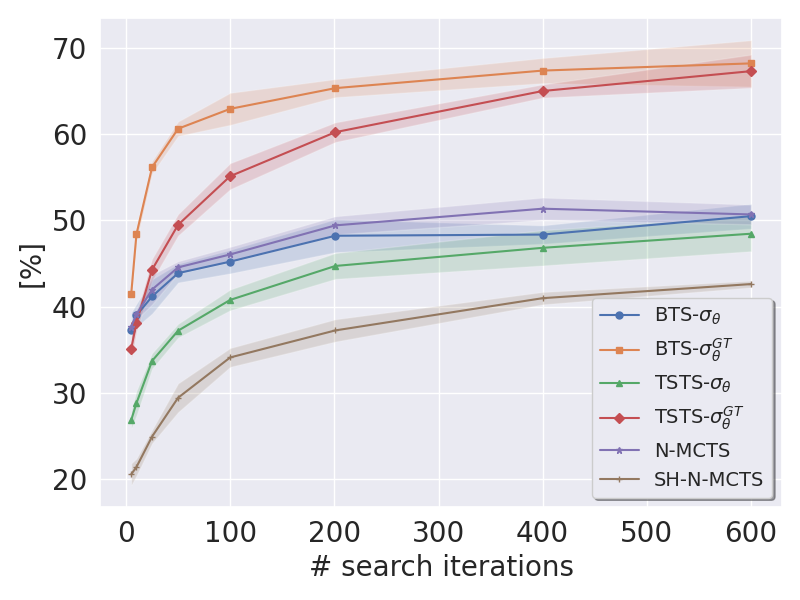}
    }
    \caption{Stochastic planners (test domains)}
    \label{fig:stochastic_planners_test}
    \end{subfigure}
    \caption{Success rate of different planners on ProcGen maze. Left + Middle: deterministic planners. Right: stochastic planners. Error bars are over 6 neural networks obtained from independent training runs. See Section \ref{ssec:experiments_maze} for more details.}
    \label{fig:planner_comparison}
    \vspace{-1em}
\end{figure*}

\begin{figure}
    \centering
    \begin{subfigure}[b]{0.23\textwidth}
    \centering
    \resizebox{\linewidth}{!}{
        \includegraphics[width=0.49\columnwidth]{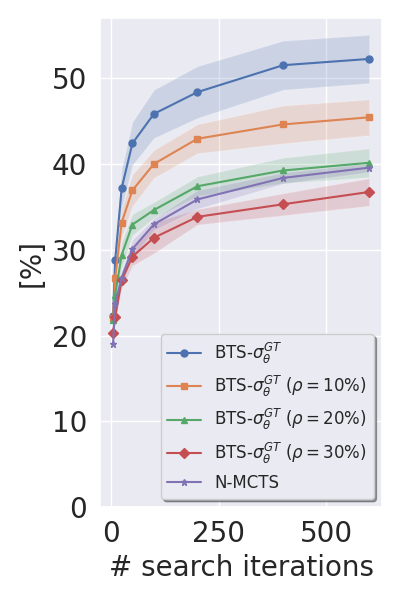}
    }
    \caption{GT error ablation}
    \label{fig:noise_ablation}
    \end{subfigure}
    \begin{subfigure}[b]{0.23\textwidth}
    \centering
    \resizebox{\linewidth}{!}{
        \includegraphics[width=0.49\columnwidth]{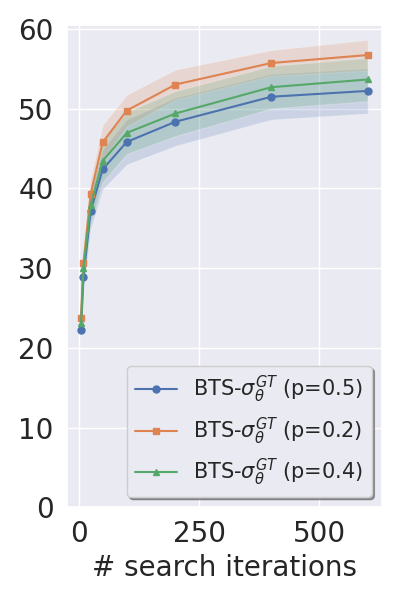}
    }
    \caption{Action commitment ablation}
    \label{fig:percentiles}
    \end{subfigure}
    \vspace{-1em}
    \caption{Ground truth uncertainty error and action commitment ablations, per Section \ref{ssec:experiments_maze} in the text.}
    \vspace{-1em}
\end{figure}


In Figure \ref{fig:planner_comparison} we compare various deterministic Bayesian planners with \nmctsalg, under different search budgets.
In this experiment, the neural network was trained using an \nmctsalg\ annotator for 250 epochs.
When training the head for predicting $\sigma_\theta(s,a)$ we kept the rest of the network frozen, therefore the \nmctsalg\ results are the same as would have been obtained without learning the uncertainty. 
Our comparison uses the same $\mu_{\theta}(s,a)$ in all planners. Furthermore, in this experiment we chose MCTS action commitment for all planners.
Thus, \textbf{the only difference between the planners is how they search the tree}. We make several observations. Clearly, the results on training domains and test domains are markedly different, validating our hypothesis that epistemic uncertainty can significantly affect planning. On training domains, all methods are comparable, and achieve significantly better results than on test, where network predictions can be significantly less accurate. We next discuss the comparison on test domains. 
First, with ground truth uncertainty, all Bayesian methods (\methodabbrv\ and \methodBabbrv\ variants) significantly outperform \nmctsalg, validating our main premise -- \textit{accounting for epistemic uncertainty during the search leads to more informative search trees} (see Appendix \ref{sec:tree_search_analysis} for a detailed analysis including illustration of the search trees). Second, we observe that \methodBabbrv\ outperforms \buctalg\ and \bucbalg, which outperform \methodabbrv. Thus, in contrast to the MAB setting~\cite{chapelle2011empirical}, it appears that in tree search Thompson sampling is significantly outperformed by (Bayesian) exploration bonuses. Third, we observe that the learned uncertainty estimates are not accurate enough to yield improvement over \nmctsalg.

In Figure \ref{fig:stochastic_planners_test} we show similar results for stochastic planners: \methodabbrv\ and \gumbelalg, in which the search process is stochastic, and also deterministic \nmctsalg\ and \methodBabbrv\ with a stochastic SoftMax action commitment. Stochasticity helps avoid recurring mistakes and improves performance, yet the relative ordering between the algorithms is similar to the deterministic case. For \gumbelalg, we noticed significantly worse performance than all other methods in this domain. In Appendix Figure \ref{fig:ensemble_comparison} we show a similar comparison and similar conclusions using neural network ensembles.


We next study how accurate uncertainty prediction needs to be to yield improvements over \nmctsalg. We add $\rho\%$ error to each ground truth uncertainty estimate in Eq.~\ref{eq:gt_uncertainty} by multiplying it by $1+U$, where $U\sim \Uniform(-0.01 \rho,0.01 \rho)$. Our results in Figure \ref{fig:noise_ablation} show that with up to $20\%$ error \methodBabbrv\  still significantly outperforms \nmctsalg. 

In Figure \ref{fig:percentiles} we investigate the action commitment during online planning, using \methodBabbrv\ with ground truth uncertainty and different quantiles (cf. Section~\ref{ssec:commitment}). We observe that committing to \textit{risk-averse} actions significantly improves performance, while the results above show that during the exploration of the search tree being risk seeking is beneficial.
Thus, maintaining posterior distributions has additional benefits in online planning beyond the improved search trees.

\vspace{-0.5em}
\section{Discussion}\label{sec:conclusion}

The hypothesis in this paper is that uncertainty estimates can benefit the search in online planning with neural networks. We developed the fundamentals of a Bayesian tree search that facilitates such estimates, and proposed several practical algorithms. Our experimental results are mixed: on the one hand, with ground truth uncertainty estimates, we observed a dramatic improvement over state-of-the-art frequentist methods. On the other hand, ground truth estimates are not practical, and our efforts to learn uncertainty estimates proved too inaccurate to yield significant gains in planning. 

Our results are specific to the ProcGen maze and leaper environments, and it is possible that domains with different reward structures and dynamics will be more or less sensitive to the uncertainty estimation accuracy. Nevertheless, we conclude that there is great potential to studying methods for estimating neural network uncertainty in the sequential decision making setting, reinforcing a similar conclusion made by \citet{riquelme2018deep} for Bayesian bandits. Some promising recent developments include methods based on conformal prediction~\cite{angelopoulos2021gentle} and epistemic neural networks~\cite{osband2023approximate}. 

Another interesting direction is to learn posteriors that depend on the complete search history, instead of the independent neural network estimates used here, for example, using transformer models that learn complete decision making strategies~\cite{laskin2022context}. The Bayesian view offers a natural framework for such methods.

\newpage

\section*{Impact Statement}
This paper presents work whose goal is to advance the field of Machine Learning. There are many potential societal consequences of our work, none which we feel must be specifically highlighted here.

\bibliographystyle{icml2024}
\bibliography{sample}

\newpage
\appendix
\onecolumn

\section{\methodBfull\ Pseudo-code}
\label{sec:bts_pseudo}

We provide a complete pseudo-code for \methodBabbrv. Differences from \methodabbrv\ are highlighted.

\begin{algorithm}[h]
\caption{\methodBfull}\label{alg:bts}
\begin{algorithmic}
    \REQUIRE $P_{\textrm{query}}(Q(s,a))$, $r_{\textrm{query}}(s,a)$, $c$.
    \STATE Init known set $S_{\textrm{known}} := \{s_0\}$
    \textcolor{blue}{\STATE Init state visit counters $N(s) := 0 $ for all $s \in S$}
    \STATE Init value $P(Q(s_0,a)) := P_{\textrm{query}}(Q(s_0,a)) \quad \forall a\in A$. 
    \FOR{t = 1,2,\dots,T} 
    \STATE \COMMENT{Forward sampling}
    \STATE Set $s:=s_0$, $s':=\emptyset$
    \WHILE{True}
        \STATE \textcolor{blue}{Set $N(s) := N(s) + 1$}
        \STATE \textcolor{blue}{Set $a^* := \argmax_a \quantile(1 - (1 - \alpha_0) \cdot e^{-\frac{N(s) - 1}{\beta}}, P(Q(s,a)))$}
        \STATE \textcolor{blue}{Set $s' := f(s,a^*)$}
        \IF{$s' \notin S_{\textrm{known}}$} 
            \STATE Set $ S_{\textrm{known}} := S_{\textrm{known}} \cup \{s'\} $
            \STATE Break
        \ENDIF
        \STATE Set $s := s'$
        \ENDWHILE
    \STATE Set value distribution for leaf $P(Q(s',a')) = P_{\textrm{query}}(Q(s',a')) \quad \forall a'\in A$.
    \STATE {}\COMMENT{Max-Backup}
    \WHILE{True}
        \STATE Set $(s,a) := f^{-1}(s')$ 
        \STATE Set $P(Q(s,a)) := P(r(s,a) + \max_{a'\in A} Q(s',a') )$
        \IF{$s = s_0$}
            \STATE Break
        \ENDIF
        \STATE Set $s' := s$
    \ENDWHILE
    \ENDFOR
    \STATE \textbf{return} $P(Q(s_0,a))$
\end{algorithmic}
\end{algorithm}

\section{Implementation Details}
\label{sec:implementation_details}
We detail technical points in our implementation of training and evaluation.

\subsection{ProcGen Maze Environment}
Throughout our implementation, we did not use a discount factor. To account for this, we modified the ProcGen reward function such that the reward for each time step until reaching the goal is $-1.0$. This induces short paths to the goal without discounting.

We used random seeds to generate different maze environments for training and testing. In particular, we use a dataset of $150$ samples for training, and also present the results of an evaluation on this set (see Figure \ref{fig:planner_comparison_train}.). For testing, we use a disjoint set of $500$ samples to evaluate the different planners (see Figure \ref{fig:planner_comparison}).

\subsection{Neural Network Training Parameters}
\label{ssec:training_params}
Our neural network model was trained using the following parameters:


The optimizer is an Adam optimizer with fixed learning rate of 0.001, $\beta$ coefficients of 0.9 and 0.999 for running averages of gradient and its square respectively, and $\epsilon$ of 1e-8.


We expand on our data collection and loading. 
We maintain a buffer of $N_{\text{buffer}}$ samples. 
Each sample contains a state $s$, and $4$ targets $Q(s,a)$, one for each action at that state. During training, in each epoch we sample $N_{bs}$ batches, each with size $32$ samples, from the buffer for training the neural network. We fill the buffer in a FIFO manner using the annotator. 
The annotator is run on $150$ different mazes, each for up to $200$ environment steps (early stopping if reaching the goal), and each step has a search budget of $250$. After each search, the annotator converts the $Q(s,a)$ values at the root (in the case of Bayesian annotator, we take the expected $Q$ values) to probabilities by using SoftMax with a temperature scaling of $10$, and samples an action commitment according to these probabilities. The committed action is used to advance the state of the world, and a new tree search is started over the new state, and this process is repeated.
The $Q$ values at each root state during the interaction are inserted as targets to the buffer.
\\
In our experiments we set $N_{\text{buffer}} = 40000$, and $N_{bs}=200$.


We will release checkpoints of the trained networks to reproduce the figures in the paper.


Figure \ref{fig:sr_during_train} shows the success rate of the different planners on a subset of the training set, during the training procedure, using a \methodBabbrv\ annotator. In the experiments in the main text we only used an \nmctsalg\ annotator.
Note that we show training results on a small subset of training domains, for fast evaluations of the planners during training, therefore the results in the figure are not comparable to the results in the main text. 

\begin{figure}[!htb]
    \centering
    \includegraphics[width=\columnwidth]{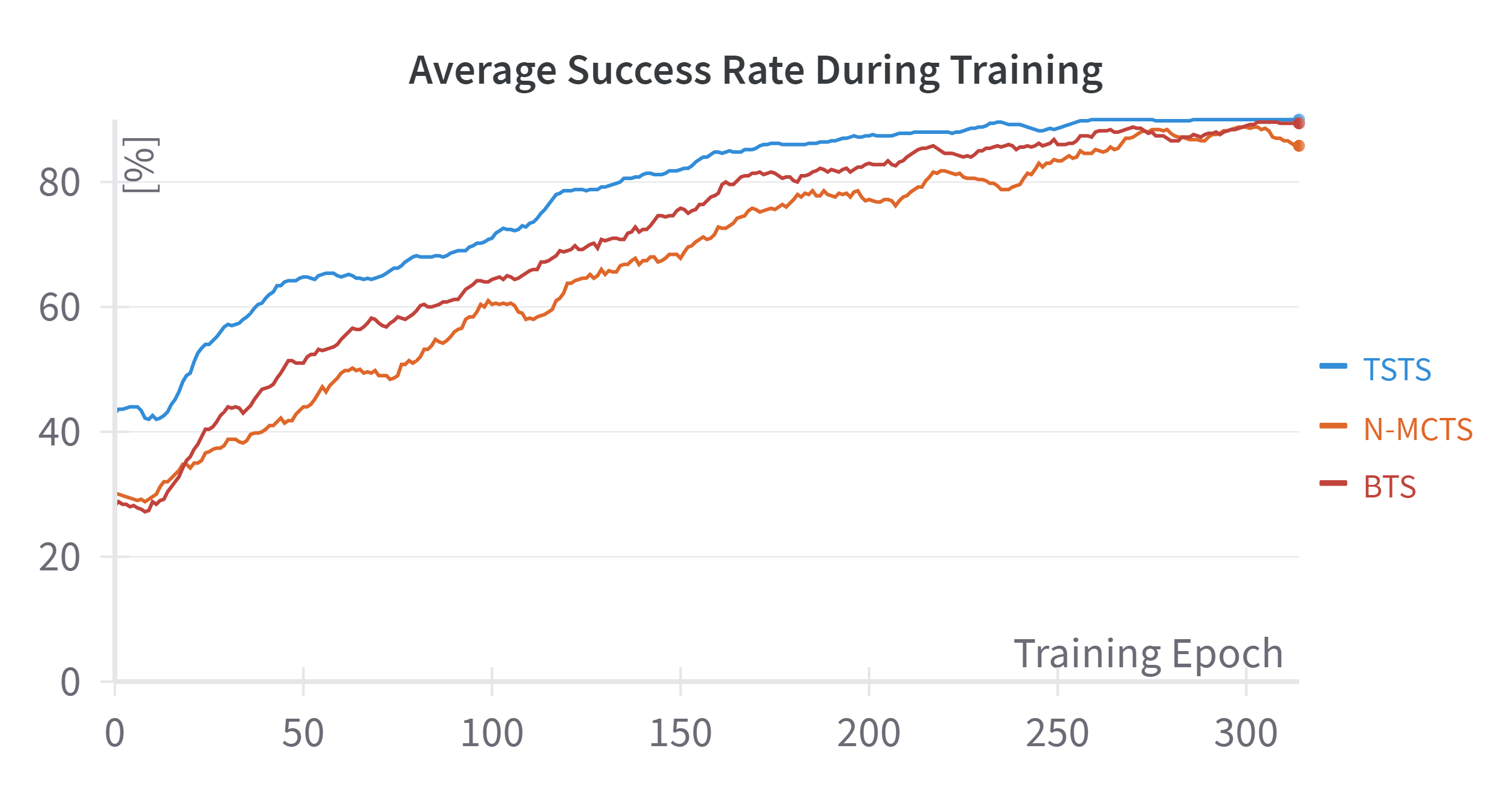}
    \caption{Success rate on train domains during the training}
    \label{fig:sr_during_train}
\end{figure}

\subsection{Computing the Max-Backup}
\label{ssec:max_backup}

The calculation of the distribution of max of several independent random variables is given by
\begin{equation*}
    P (\max(X_1, ..., X_n) \leq a) = P (X_1 \leq a \text{ ... and } X_n \leq a) = P (X_1 \leq a , \text{ ... } X_n \leq a) = \prod_{i=1}^n P(X_i \leq a),
\end{equation*}
where in the last equality we use the independence assumption.

To perform this calculation in practice we use Algorithm \ref{alg:max_backup} with $M=50$ bins linearly spaced starting for each individual CDF from $0.001$ and ending at $0.999$, which in the case of normal distribution covers $\pm3$\ standard deviations around the mean of the distribution.
\\
\\

\begin{algorithm}[!htb]
\caption{Max-Backup}\label{alg:max_backup}
\begin{algorithmic}
    \REQUIRE individual cdf distributions $\left\{C_i\right\}_{i=1}^N$, given at $M$ bins $b_i=\{b_1^i$, \dots, $b_M^i$\} for each $C_i$.

    \STATE $\textrm{first\_bin} := \max_i{(b_1^i)}$
    \STATE $\textrm{last\_bin} := \max_i{(b_M^i)}$
    \STATE $\textrm{all\_bins} :=$ linspace(\textrm{first\_bin}, \textrm{last\_bin}, $M$)
    \STATE
    \FOR{i = 1,2,\dots,N}
        \STATE $C_i^{interp} := \textrm{interpolate}(C_i \text{ over points } \textrm{all\_bins})$
    \ENDFOR
    \STATE
    \STATE $\textrm{MaxDistributionCdf} := \prod_{i=1}^N C_i^{interp}$
    
    \STATE \textbf{return} $\textrm{all\_bins}, \textrm{MaxDistributionCdf}$
\end{algorithmic}
\end{algorithm}

\subsection{Approximations}
\label{ssec:gauss_apprx}
When performing the forward sampling in TSTS one has to sample a general pdf $P(Q)$ resulting from the distribution of the max-backup.
An approximate calculation is to sample a Gaussian distribution with the expectation and variance of $P(Q)$. Comparing the performance we see little difference between the exact and approximate sampling, so we opted to use the approximate sampling.

We adopt a similar approximation also for \methodBabbrv\ where instead of calculating the exact percentile  $\quantile(\qalpha, P)$ for a general $P(Q)$ we calculate it assuming a Gaussian distribution with the expectation and variance of $P(Q)$. Comparison with the exact calculation showed little difference in results.

\section{Standard Deviation Estimation via Ensemble of Neural Networks}
\label{sec:ensemble}
In this section we report in Figure \ref{fig:ensemble_comparison} the results using an ensemble of neural networks to estimate the uncertainty (see Section \ref{ssec:learning_bayesian_tree_search}) for details.
For this evaluation, we used an ensemble $\numensemble=5$ models, where all models are initialized with a random set of weights, and trained similarly.

The results using this method are similar to the ones reported in \ref{ssec:experiments_maze}.
We do note that using the ensemble improves the predictions of $Q(s,a)$ (by taking the average of the different models), which mainly improves the performance of \nmctsalg\ on test domains.

\begin{figure*}[!htb]
    \begin{subfigure}[b]{0.5\textwidth}
    \centering
    \resizebox{\linewidth}{!}{
        \includegraphics[width=\columnwidth,height=0.9\columnwidth]{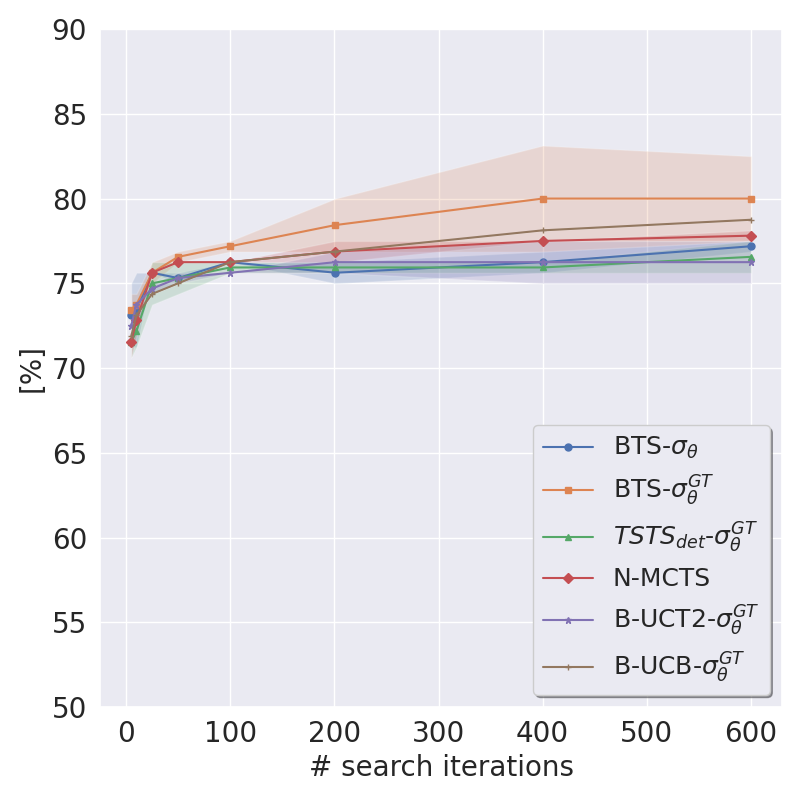}
    }
    \caption{Train domains}
    \end{subfigure}
    \begin{subfigure}[b]{0.5\textwidth}
    \centering
    \resizebox{\linewidth}{!}{
        \includegraphics[width=\columnwidth,height=0.9\columnwidth]{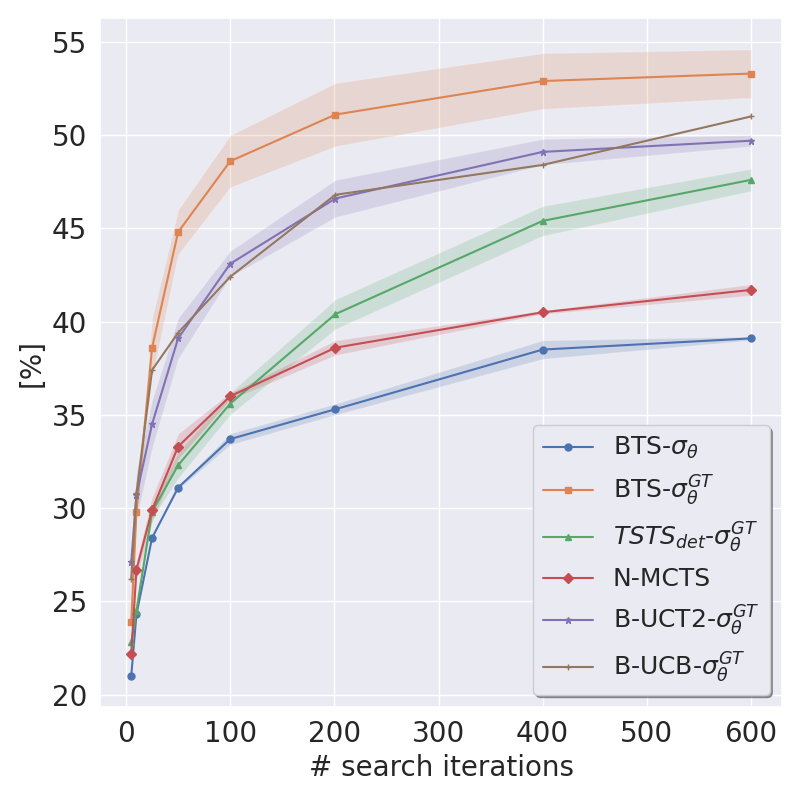}
    }
    \caption{Test domains}
    \end{subfigure}
   
    \caption{Comparison of Different Planners Using an Ensemble of NN }
    \label{fig:ensemble_comparison}
\end{figure*}

\newpage
\section{Evaluation on ProcGen Leaper}
\label{sec:leaper_results}
In this section we describe additional results on another ProcGen environment, the Leaper environment.
In this game, depicted in Figure \ref{fig:leaper}, a frog starts at the bottom of the screen and must get to the finish line.
To achieve that, it has to pass road lanes while not being hit by passing cars, and then cross the river by jumping onto wooden logs.
If the frog is hit by a car or falls into the river the game terminates, while reaching the finish line yields a reward of $10$.
Since we set a discount factor of $1$ in all of our experiments, we modified the Leaper default reward function such that an \textit{Up} action receives a reward of $-0.1$, while all other actions receive a reward of $-0.2$ each time step until the game terminates.
This modification is done to encourage the agent to reach the goal quickly, instead of stalling and assuming it will obtain the same accumulated reward in the future.
Unlike the maze, this environment is dynamic, hence the frog also has a \textit{Wait} action in addition to moving in four directions.

The original Leaper domain exhibited a very small train-test gap in previous work~\cite{cobbe2020leveraging}. To generate a significant train-test gap, we manually divided Leaper instances into train and test datasets, such that train instances have at most two road lanes and two river lanes (see an example in Figure \ref{fig:leaper_train}), while test instances can have more (Figure \ref{fig:leaper_test}).
We used $N_{train}=100$ for training of the neural network and $N_{test}=100$ for evaluation. 
We trained the NN for $60$ epochs in a similar manner to the described in Section \ref{ssec:training_params}.

We tested {\methodBabbrv} against {\nmctsalg} on the test instances, where each planner is evaluated for $k=25$ time steps, where at each time step a search is performed and an action is committed and executed, according to the online planning scheme described in Section \ref{sec:bayesian_online_planning}.
To solve an instance, the frog must reach the finish line within the $k=25$ time steps (and obviously not get hit by a car or fall into the river before that).
To estimate GT uncertainty values for {\methodBabbrv} we use an $A^*$ search from each state where the vertical distance of the agent from the finish line is used as an admissible heuristic.

Figure \ref{fig:leaper_test_results} shows the success rate (i.e., the percentage of solved instances) of {\methodBabbrv} with and without GT uncertainty estimates against \nmctsalg. 
Even without GT estimates, {\methodBabbrv} significantly outperforms \nmctsalg, where when incorporating GT, {\methodBabbrv} can achieve a success rate $>85\%$ on the test set.

Notably, the neural network uncertainty estimates on Leaper are good enough to yield a significant improvement, differently from the maze domain. We explain this by observing that in Leaper, some of the uncertainty is aleatoric, for example, the uncertainty about whether a log is going to be spawned in the next frame or not. This uncertainty is similar in training and testing, and is easier to capture by training the neural network using the MLE loss. 

\begin{figure*}[!htb]
    \centering
    \begin{subfigure}[b]{0.3\textwidth}
    \centering
    \resizebox{\linewidth}{!}{
        \includegraphics{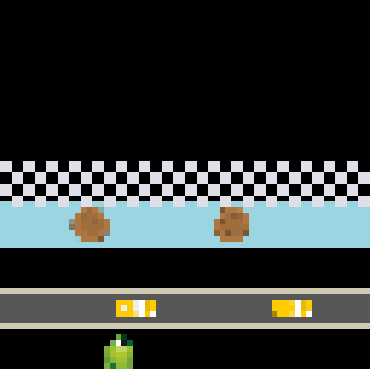}
    }
    \caption{Leaper typical train instance}
    \label{fig:leaper_train}
    \end{subfigure}
    \begin{subfigure}[b]{0.3\textwidth}
    \centering
    \resizebox{\linewidth}{!}{
        \includegraphics{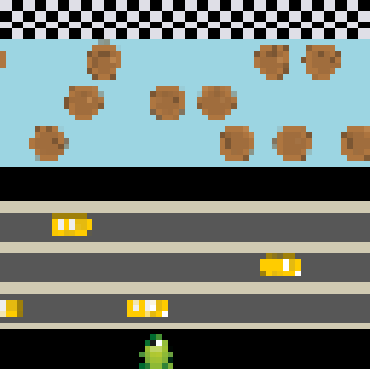}
    }
    \caption{Leaper typical test instance}
    \label{fig:leaper_test}
    \end{subfigure}
   
    \caption{ProcGen Leaper environment}
    \label{fig:leaper}
\end{figure*}

\begin{figure}[!htb]
    \centering
    \includegraphics[width=\columnwidth]{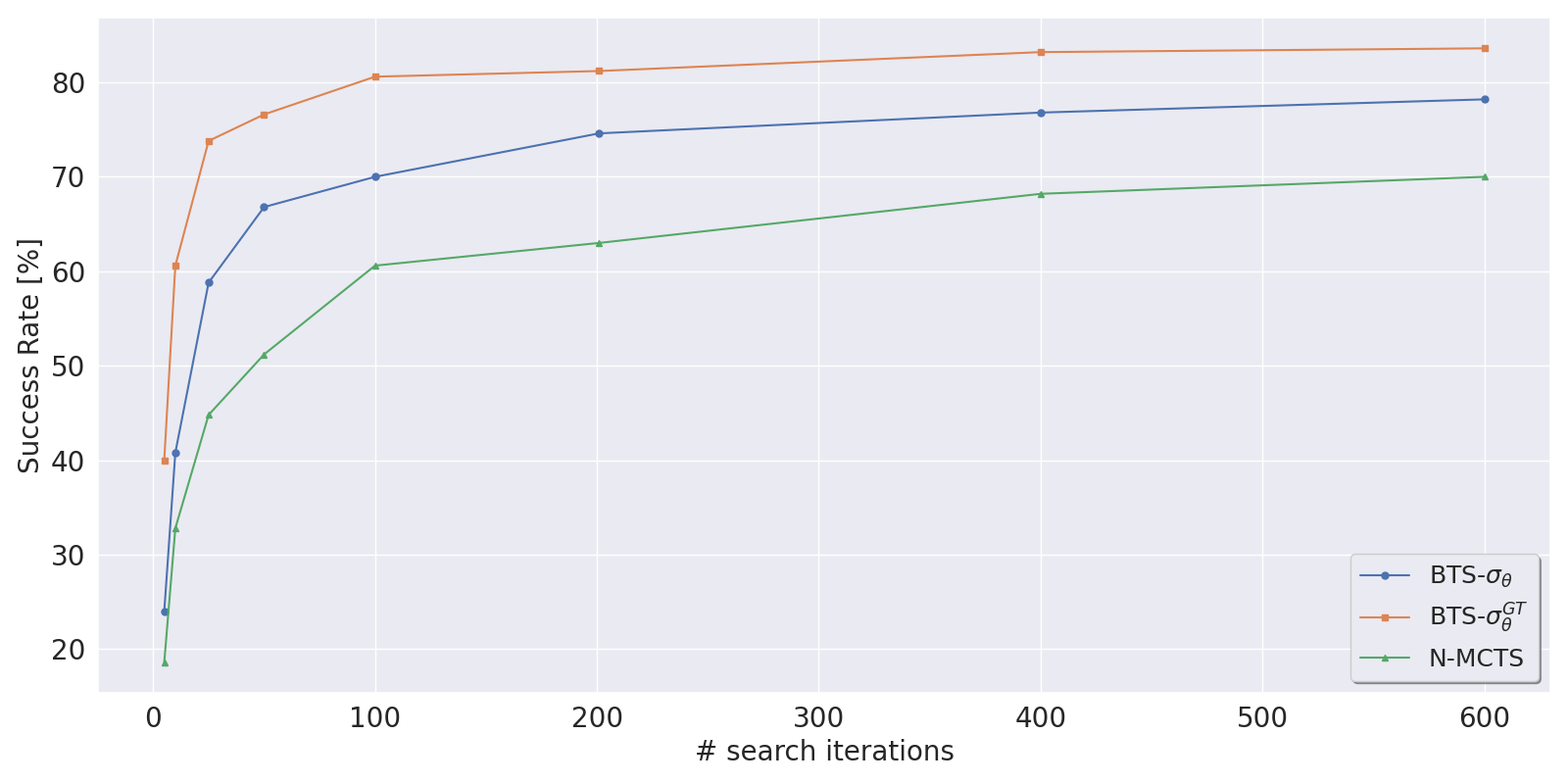}
    \caption{Test Results on ProcGen Leaper}
    \label{fig:leaper_test_results}
\end{figure}

\newpage
\section{Sensitivity to Hyper-Parameters}
\label{sec:senitivity}
In this section we provide the results of an ablation studying the hyper parameters for the different planners, and choosing the best ones for each planner for a fare comparison.

\subsection{\bucbalg}
The \bucbalg\ algorithm selects actions to explore in the tree search, based on quantiles of the $Q(s,a)$ posterior distributions, according to formula \ref{eq:action_selection_percentile}.
The quantile level $\qalpha(s)$ depends on the number of visits $N(s)$ and given by:
\begin{equation*}
    \qalpha(s) = 1 - \frac{\beta}{N(s)}.
\end{equation*}

We tested several values of $\beta$ as  depicted in Figure \ref{fig:bucb_robustness}, and found out that $\beta=0.5$ gives the best results.

\begin{figure}[!htb]
    \centering
    \includegraphics[width=\columnwidth]{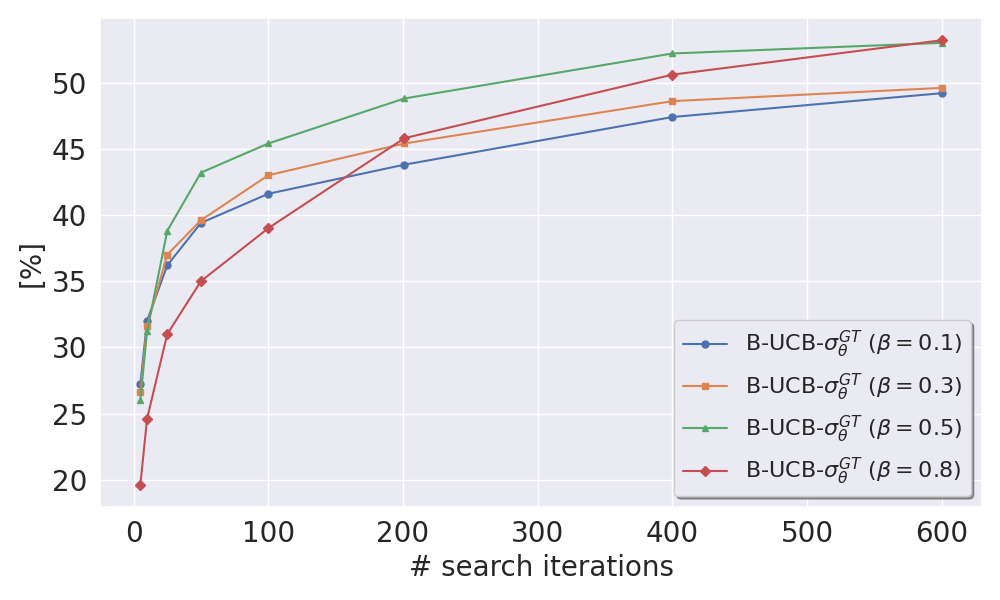}
    \caption{\bucbalg\ $\beta$ ablation}
    \label{fig:bucb_robustness}
\end{figure}

\subsection{\methodBabbrv}
Our suggested \methodBabbrv\ algorithm also selects actions based on quantiles of the $Q(s,a)$ posterior distributions, according to formula \ref{eq:action_selection_percentile}.
However, here the quantile level $\qalpha(s)$ depends on two hyper-parameters $\qalpha_0$ and $\beta$ by:
\begin{equation*}
    \qalpha(s) = 1 - (1 - \alpha_0) \cdot e^{-\frac{N(s) - 1}{\beta}}.
\end{equation*}

We tested the cross product of the following values for each parameter: $\qalpha_0=[0.1, 0.3, 0.5, 0.8]$, $\beta=[0.5, 1, 3, 8]$ (i.e., 16 different choices) and report the results in Figure \ref{fig:bts_robustness}.
We found out that $\qalpha_0=0.5$, $\beta=3$ gives the best results.

\begin{figure}[!htb]
    \centering
    \includegraphics[width=0.8\columnwidth,height=0.5\columnwidth]{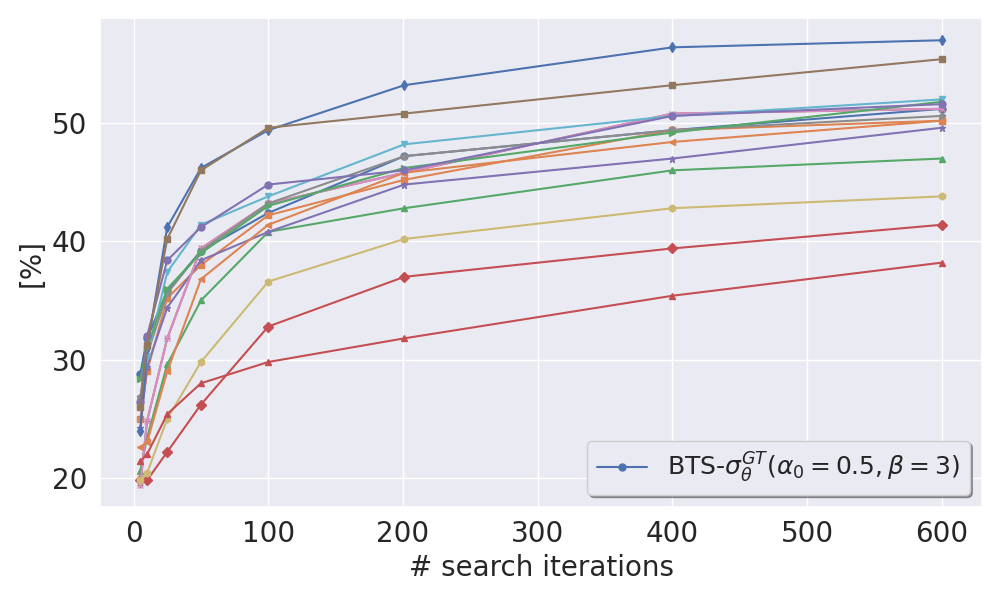}
    \caption{\methodBabbrv\ $\qalpha_0$ and $\beta$ ablation}
    \label{fig:bts_robustness}
\end{figure}

\subsection{\nmctsalg\ with stochastic SoftMax action commitment}
In SoftMax action commitment, we choose an action to perform in the environment by converting $Q(s,a)$ values to probabilities using a SoftMax operation.
We tested several values for the SoftMax temperature and report the results in Figure \ref{fig:nmcts_temp_robustness}.
We found that using a temperature of $2.0$ yields the best performance.

\begin{figure}[!htb]
    \centering
    \includegraphics[width=0.8\columnwidth,height=0.6\columnwidth]{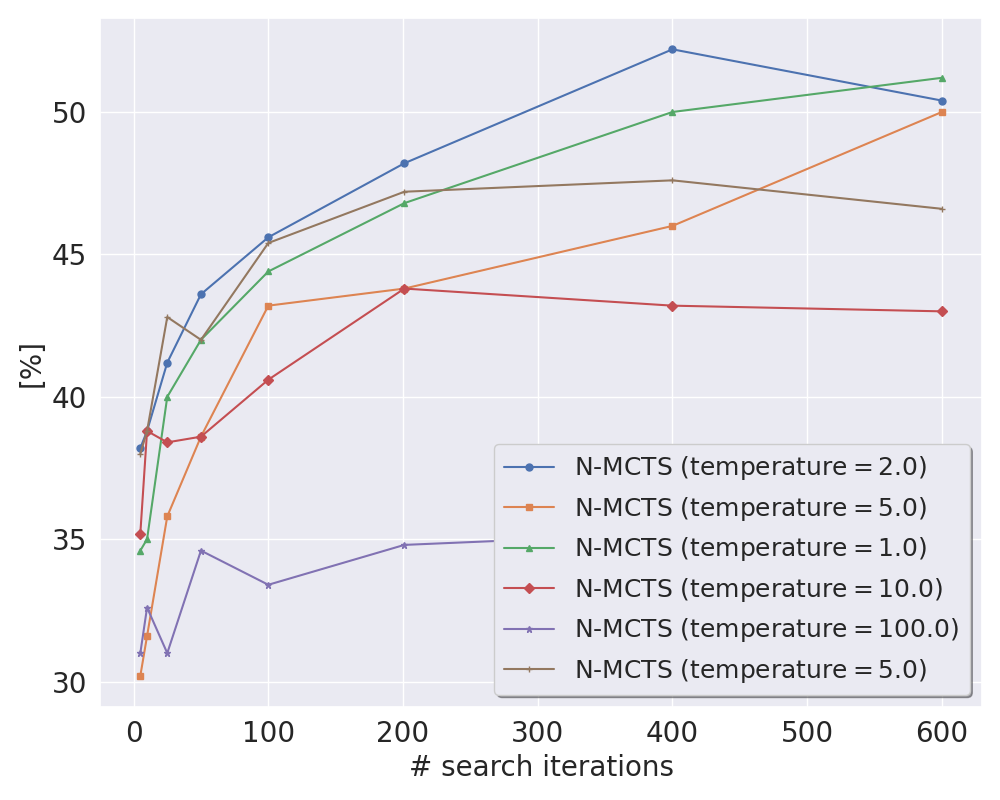}
    \caption{\nmctsalg\ sensitivity to SoftMax temperature}
    \label{fig:nmcts_temp_robustness}
\end{figure}

\subsection{\methodBabbrv\ with stochastic SoftMax action commitment}
Similarly, we tested several values of the temperature when using SoftMax action commitment with \methodBabbrv.
Results are shown in Figure \ref{fig:bts_temp_robustness}.
We found that the best temperature value in this case is $1.0$, however for simplicity we used the value $2.0$ here as well (similar to \nmctsalg\ with SoftMax action commitment), and note that using a value of $1.0$ would improve the results of \methodBabbrv\ shown in Figure \ref{fig:stochastic_planners_test}. 

\begin{figure}[!htb]
    \centering
    \includegraphics[width=\columnwidth]{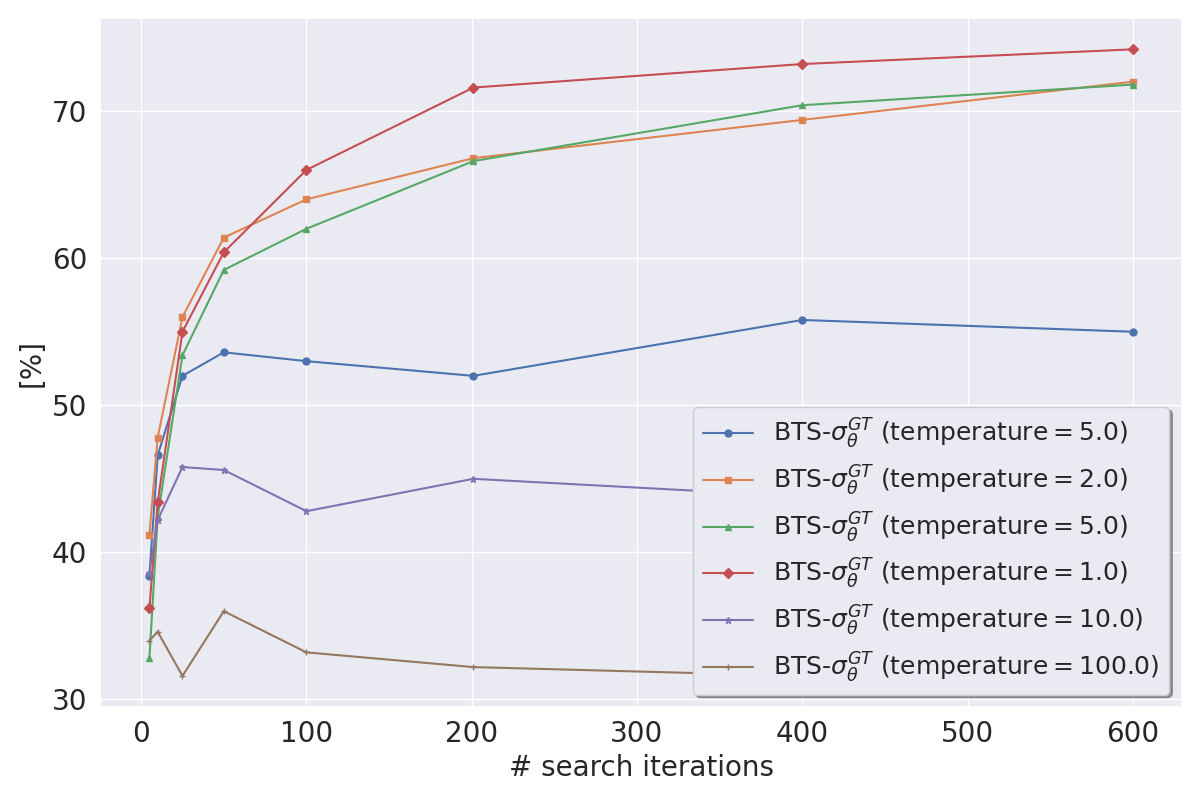}
    \caption{\methodBabbrv\ sensitivity to SoftMax temperature}
    \label{fig:bts_temp_robustness}
\end{figure}

\subsection{Comparison with {\wmcts} and \dng}
\label{ssec:comparison_wmcts_dng}
In this section we provide a comparison between our suggested algorithms, and two previously suggested methods that share similar ideas to ours.
The first method is the Dirichlet-NormalGamma MCTS ({\dng}) algorithm suggested by \citet{bai2013bayesian}.
In \dng, the selection rule of each node of the MCTS search tree is replaced with Thompson sampling, assuming that $Q$ values are distributed as a mixture-of-Gaussians, i.e., $~N(\mu, \frac{1}{\tau})$, where $\tau=\frac{1}{\sigma^2}$ and $(\mu,\tau)$ follows a NormalGamma distribution defined by parameters $\langle \mu_0, \lambda, \alpha, \beta \rangle$ and a pdf of the form:
\begin{equation*}
    f(\mu,\tau|\mu_0,\lambda,\alpha,\beta)=\frac{\beta^\alpha \sqrt{\lambda}}{\Gamma(\alpha)\sqrt{2\pi}}\tau^{\alpha-\frac{1}{2}}e^{-\beta\tau}e^{-\frac{\lambda\tau(\mu-\mu_0)^2}{2}},
\end{equation*}
where $\Gamma(\alpha)$ is the Gamma function.
{\dng} then uses the following equations in the backup procedure, given the backed-up value $r$:
\begin{equation}
\label{eq:dng_backup}
\begin{split}
    \alpha & \leftarrow \alpha + 0.5
    \\
    \beta & \leftarrow \beta + (\lambda(r-\mu_0)^2/(\lambda+1))/2
    \\
    \mu_0 & \leftarrow (\lambda\mu_0+r)/(\lambda+1)
    \\
    \lambda & \leftarrow \lambda +1.
\end{split}
\end{equation}
More specifically, since the evaluated maze environment is deterministic, the $Q$ value at each node is distributed as a single Gaussian, hence we don't need the Dirichlet distribution for sampling the weight of each Gaussian, as described in the original {\dng} paper.
The implementation of {\dng} is then almost identical to {\nmctsalg}, except the following changes: a node's value is sampled from the NormalGamma distribution given it's current $\langle \mu_0, \lambda, \alpha, \beta \rangle$ values, and a backup for updating these values is performed using equations \ref{eq:dng_backup}.
In addition, for any un-visited node, its value is approximated using the NN, instead of the simulation using a rollout policy done in the original {\dng} algorithm.
Following the suggestion in \citet{bai2013bayesian}, for each node in {\dng}, $\alpha$ is initialized to $1$, and $\mu$ is initialized to $0$ (we also tried initializing $\mu$ from the NN and found out it performed worse).
We performed a hyper-parameter tuning and found that initializing $\beta=100$ and $\lambda=0.001$ performed best in the maze setting.

The second method we compared against is Wasserstein MCTS (\wmcts), suggested by \citet{dam2023monte}.
{\wmcts} models value distributions as Gaussians similar to our work, and propagates the uncertainty of the estimate of value nodes across the tree using a backup operator that computes value nodes as Wasserstein barycenters of children action-value nodes. 
Both an Optimistic Selection (similar to the UCT formula by replacing exploration term by the standard deviation) and Thompson sampling are proposed as action-selection strategies.
We found that using Thompson sampling for action selection in {\wmcts} outperforms using Optimistic Selection, therefore we compare it to our {\methodabbrv} algorithm.
For both algorithms ({\wmcts} and \methodabbrv) we use the deterministic variant (see Section \ref{sec:experiments}).
In practice, the implementation of {\wmcts} with Thompson sampling is directly derived from our {\methodabbrv} implementation by only replacing the backup method.
To set the $p$ parameter we performed a hyper-parameter tuning and found that $p=1$ worked best in the maze setting.

We evaluated both methods on the test dataset of the ProcGen maze environment, using the same neural network used in our method.

Our results show (Figure \ref{fig:tsts_vs_dng}) that {\dng} (with the neural network value backup) is comparable to standard {\nmctsalg} (though much worse when search budget is small).
Since {\dng} is a stochastic algorithm, we compare it to {\methodabbrv} and observe that it performs worse.
With GT uncertainty, {\methodabbrv} significantly outperforms {\dng}.

Additionally, Figure \ref{fig:tsts_vs_wmcts} shows that {\wmcts} is comparable (but slightly worse on most search iterations) to {\methodabbrv} without GT uncertainty, but significantly outperformed by {\methodabbrv} with GT uncertainty for all search iterations. 
We hypothesize that our Max-Backup, which was naturally derived for this task, is the reason for this result.



\begin{figure}[!htb]
    \centering
    \includegraphics[width=\columnwidth, height=0.5\columnwidth]{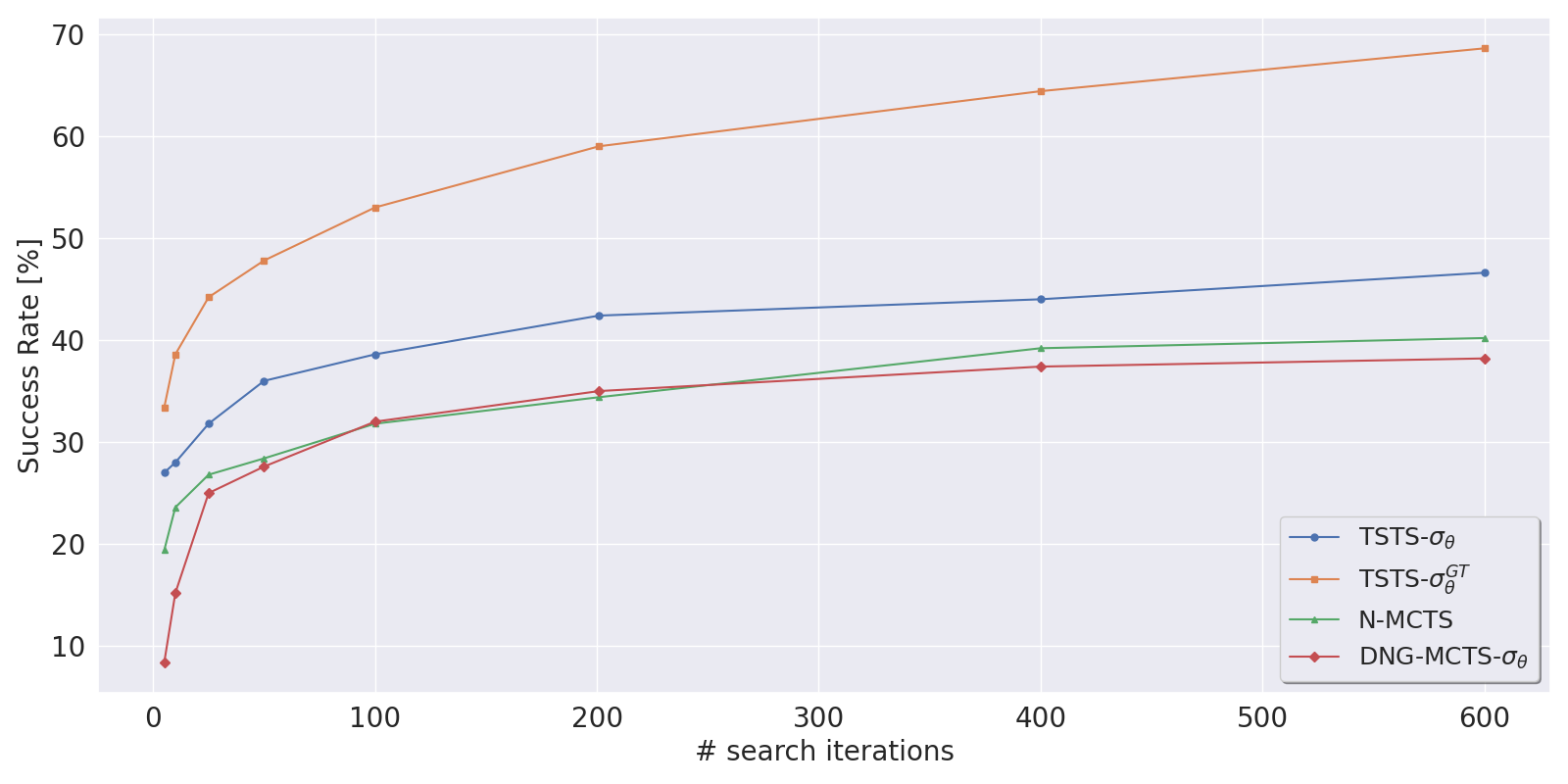}
    \caption{{\methodabbrv} vs. \dng}
    \label{fig:tsts_vs_dng}
\end{figure}

\begin{figure}[!htb]
    \centering
    \includegraphics[width=\columnwidth]{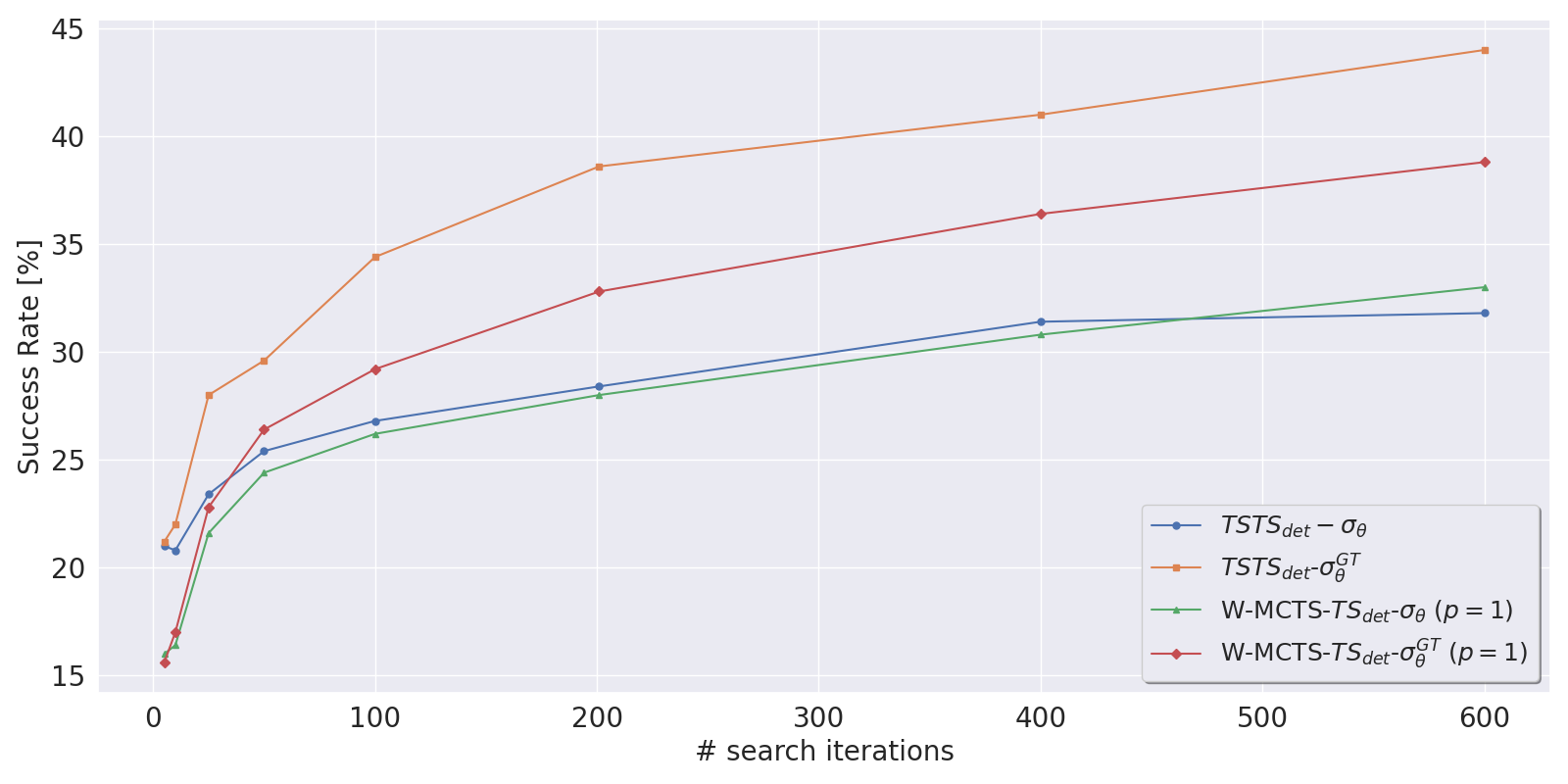}
    \caption{{\methodabbrv} vs. \wmcts}
    \label{fig:tsts_vs_wmcts}
\end{figure}


\newpage
\newpage
\section{Tree Search Analysis}
\label{sec:tree_search_analysis}
In this section we emphasize the potential of our suggested \methodBabbrv\ algorithm compared to \nmctsalg.
Specifically, we show how \methodBabbrv\ exploits the uncertainty of the NN (ground truth in this example) to perform a more informed exploration of the search tree, allowing it to explore the tree to a significant depth, eventually leading to finding a rewarded state.

Figure \ref{fig:maze} shows an example maze, where the cheese (depicted in yellow) is $11$ steps away from the mouse (depicted in gray).
Figure \ref{tab:predicted_qsa_example} shows the predicted $Q(s,a)$ values from the root output by the NN.
We see that the NN predictions are pretty accurate in this case, implying that Up is the correct action in this state.
Both planners are given a search budget of $25$ iterations in this example.

Figures \ref{fig:nmcts_tree} and \ref{fig:bayesian_tree} show the trees explored by \nmctsalg\ and \methodBabbrv, respectively.
\nmctsalg\ follows the UCB formula, causing it to explore the tree in a balanced manner, since there are no significant differences in the $Q(s,a)$ values of the actions from the root.
\methodBabbrv\, on the other hand, exploits the uncertainty (or more accurately, the certainty) in the predictions to explore the tree to a much more significant depth, by always selecting the correct (Up) action from the root, and similarly at following nodes.
In other words, all other actions from the root, other than Up, yield low $Q(s,a)$ values with a very high certainty, causing \methodBabbrv\ to only exploit the correct action.
This  eventually lead to finding a terminal (rewarded) state (at leaf node $17$).
We observe similar behavior in the Leaper environment as well.

\begin{figure}[!htb]
  \begin{minipage}[b]{.5\linewidth}
    \centering
    \includegraphics[width=\columnwidth/3]{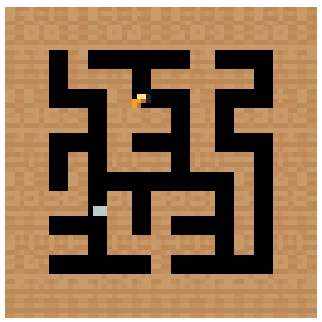}
    \caption{Example maze}
    \label{fig:maze}
  \end{minipage}\hfill
  \begin{minipage}[b]{.5\linewidth}
    \centering
    \begin{tabular}{|c | c | c | c |} 
 \hline
 Action & Predicted $Q(s,a)$ & $Q^{\text{GT}}(s,a)$ & $\sigma^{\text{GT}}_\theta(s,a)$ \\ [0.5ex] 
 \hline\hline
 Up & $-1.278$ & $-1.0$ & $0.278$ \\ 
 \hline
 Down & $-3.03$ & $-3.0$ & $0.03$ \\
 \hline
 Right & $-2.197$ & $-2.0$ & $0.197$ \\
 \hline
 Left & $-2.375$ & $-2.0$ & $0.375$ \\ [1ex] 
 \hline
\end{tabular}
\caption{Predicted $Q(s,a)$, $Q(s,a)^{GT}$ and $\sigma^{\text{GT}}_\theta(s,a)$ values at the root node}
\label{tab:predicted_qsa_example}
  \end{minipage}
\end{figure}

\begin{figure}[!htb]
    \centering
    \includegraphics[width=\columnwidth]{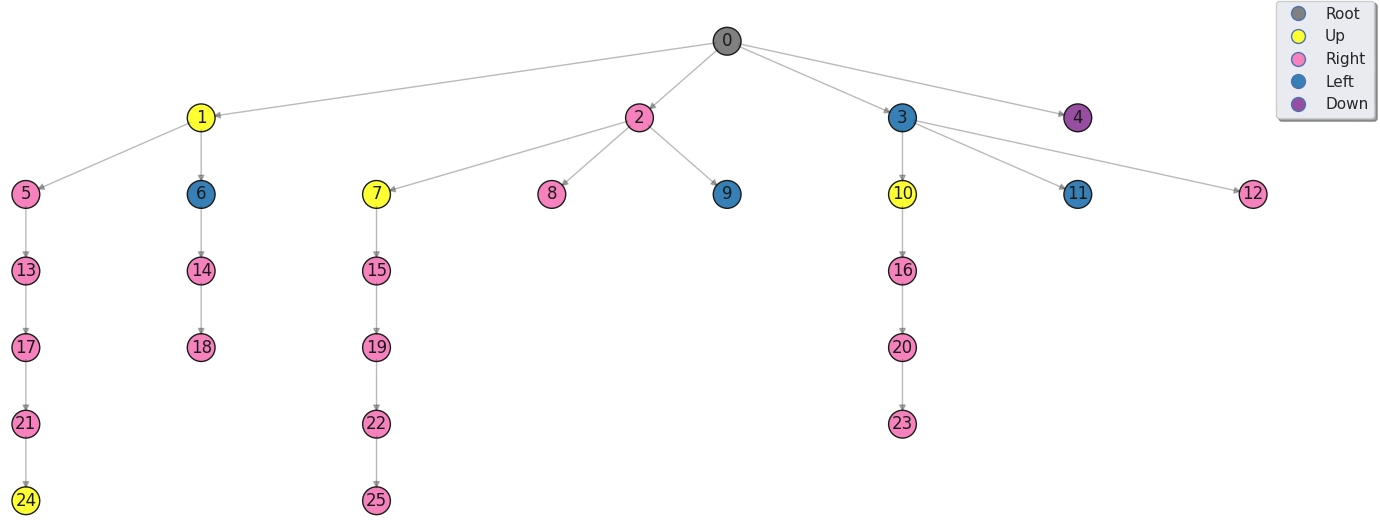}
    \caption{Tree opened by {\nmctsalg} in the Maze environment}
    \label{fig:nmcts_tree}
\end{figure}

\begin{figure}[!htb]
    \centering
    \includegraphics[width=\columnwidth]{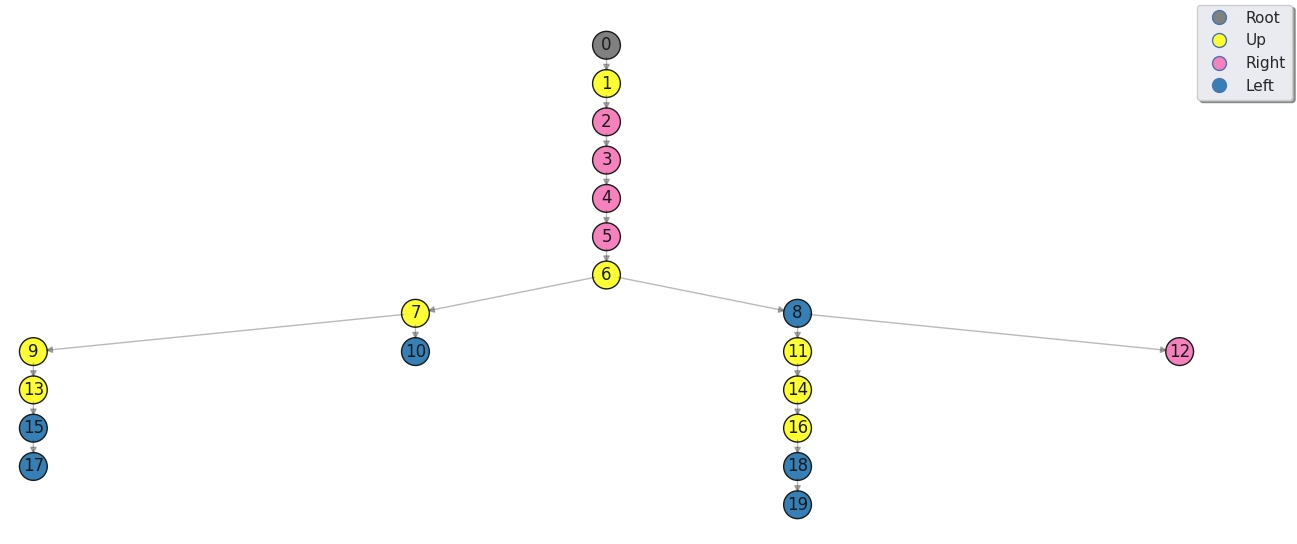}
    \caption{Tree opened by {\methodBabbrv} in the Maze environment}
    \label{fig:bayesian_tree}
\end{figure}

\newpage
\section{Proofs}\label{sec:proof}
\input{proof.tex}


\end{document}

%% file: defs.tex
\DeclareMathOperator*{\argmax}{arg\,max}
\DeclareMathOperator\erf{erf}
\DeclareMathOperator\Uniform{Uniform}

\newtheorem{theorem}{Theorem}

\newtheorem{example}{Example}
\newtheorem{proposition}{Proposition}
\newtheorem{lemma}{Lemma}

\newcommand{\methodfull}{Thompson Sampling Tree Search}
\newcommand{\methodabbrv}{\texttt{TSTS}}

\newcommand{\methodBfull}{Bayes-UCB Tree Search}
\newcommand{\methodBabbrv}{\texttt{BTS}}
\newcommand{\wmcts}{\texttt{W-MCTS}}
\newcommand{\dng}{\texttt{DNG-MCTS}}

\newcommand{\tree}{\mathcal{T}}
\newcommand{\hor}{H}
\newcommand{\ent}{\mathcal{H}}
\newcommand{\infor}{\mathcal{I}}

\newcommand{\leaves}{\mathcal{Z}}
\newcommand{\leaf}{z}

\newcommand{\roota}{A_{\mathrm{root}}}

\newcommand{\hist}{\mathcal{F}}

\newcommand{\obs}{O}

\newcommand{\loss}{\mathcal{L}}

\newcommand{\level}{\ell}
\newcommand{\trainlevels}{N_{\textrm{train}}}
\newcommand{\testlevels}{N_{\textrm{test}}}
\newcommand{\nmctsalg}{\texttt{N-MCTS}}
\newcommand{\gumbelalg}{\texttt{SH-N-MCTS}}
\newcommand{\buctalg}{\texttt{B-UCT2}}
\newcommand{\bucbalg}{\texttt{B-UCB}}
\newcommand{\quantile}{\rho}
\newcommand{\qalpha}{\alpha}
\newcommand{\numensemble}{K}

%% file: tikz/intro_schematic.tex
\tikzset{every picture/.style={line width=0.75pt}} 

\begin{tikzpicture}[x=0.75pt,y=0.75pt,yscale=-1,xscale=1]

\draw   (283.8,114.9) .. controls (283.8,103.91) and (292.71,95) .. (303.7,95) .. controls (314.69,95) and (323.6,103.91) .. (323.6,114.9) .. controls (323.6,125.89) and (314.69,134.8) .. (303.7,134.8) .. controls (292.71,134.8) and (283.8,125.89) .. (283.8,114.9) -- cycle ;
\draw   (223.4,185.3) .. controls (223.4,174.31) and (232.31,165.4) .. (243.3,165.4) .. controls (254.29,165.4) and (263.2,174.31) .. (263.2,185.3) .. controls (263.2,196.29) and (254.29,205.2) .. (243.3,205.2) .. controls (232.31,205.2) and (223.4,196.29) .. (223.4,185.3) -- cycle ;
\draw   (343.4,185.3) .. controls (343.4,174.31) and (352.31,165.4) .. (363.3,165.4) .. controls (374.29,165.4) and (383.2,174.31) .. (383.2,185.3) .. controls (383.2,196.29) and (374.29,205.2) .. (363.3,205.2) .. controls (352.31,205.2) and (343.4,196.29) .. (343.4,185.3) -- cycle ;
\draw    (291.2,131.4) -- (253.89,164.86) ;
\draw [shift={(252.4,166.2)}, rotate = 318.11] [color={rgb, 255:red, 0; green, 0; blue, 0 }  ][line width=0.75]    (10.93,-3.29) .. controls (6.95,-1.4) and (3.31,-0.3) .. (0,0) .. controls (3.31,0.3) and (6.95,1.4) .. (10.93,3.29)   ;
\draw    (318,130.6) -- (351.81,165.56) ;
\draw [shift={(353.2,167)}, rotate = 225.96] [color={rgb, 255:red, 0; green, 0; blue, 0 }  ][line width=0.75]    (10.93,-3.29) .. controls (6.95,-1.4) and (3.31,-0.3) .. (0,0) .. controls (3.31,0.3) and (6.95,1.4) .. (10.93,3.29)   ;

\draw (190,125) node [anchor=north west][inner sep=0.75pt]   [align=left] {$Q = 15 \pm 10$};
\draw (345,125) node [anchor=north west][inner sep=0.75pt]   [align=left] {$Q = 20 \pm 2$};

\end{tikzpicture}

%% file: tikz/formulation_schematic.tex
\tikzset{every picture/.style={line width=0.75pt}} 

\begin{tikzpicture}[x=0.75pt,y=0.75pt,yscale=-1,xscale=1]

\draw  [fill={rgb, 255:red, 255; green, 255; blue, 255 }  ,fill opacity=1 ] (100.8,109.9) .. controls (100.8,98.91) and (109.71,90) .. (120.7,90) .. controls (131.69,90) and (140.6,98.91) .. (140.6,109.9) .. controls (140.6,120.89) and (131.69,129.8) .. (120.7,129.8) .. controls (109.71,129.8) and (100.8,120.89) .. (100.8,109.9) -- cycle ;
\draw  [fill={rgb, 255:red, 255; green, 255; blue, 255 }  ,fill opacity=1 ] (230,59.9) .. controls (230,48.91) and (238.91,40) .. (249.9,40) .. controls (260.89,40) and (269.8,48.91) .. (269.8,59.9) .. controls (269.8,70.89) and (260.89,79.8) .. (249.9,79.8) .. controls (238.91,79.8) and (230,70.89) .. (230,59.9) -- cycle ;
\draw  [fill={rgb, 255:red, 255; green, 255; blue, 255 }  ,fill opacity=1 ] (380.4,109.9) .. controls (380.4,98.91) and (389.31,90) .. (400.3,90) .. controls (411.29,90) and (420.2,98.91) .. (420.2,109.9) .. controls (420.2,120.89) and (411.29,129.8) .. (400.3,129.8) .. controls (389.31,129.8) and (380.4,120.89) .. (380.4,109.9) -- cycle ;
\draw  [fill={rgb, 255:red, 255; green, 255; blue, 255 }  ,fill opacity=1 ] (40.4,180.3) .. controls (40.4,169.31) and (49.31,160.4) .. (60.3,160.4) .. controls (71.29,160.4) and (80.2,169.31) .. (80.2,180.3) .. controls (80.2,191.29) and (71.29,200.2) .. (60.3,200.2) .. controls (49.31,200.2) and (40.4,191.29) .. (40.4,180.3) -- cycle ;
\draw  [fill={rgb, 255:red, 255; green, 255; blue, 255 }  ,fill opacity=1 ] (160.4,180.3) .. controls (160.4,169.31) and (169.31,160.4) .. (180.3,160.4) .. controls (191.29,160.4) and (200.2,169.31) .. (200.2,180.3) .. controls (200.2,191.29) and (191.29,200.2) .. (180.3,200.2) .. controls (169.31,200.2) and (160.4,191.29) .. (160.4,180.3) -- cycle ;
\draw  [fill={rgb, 255:red, 255; green, 255; blue, 255 }  ,fill opacity=1 ] (320.4,179.9) .. controls (320.4,168.91) and (329.31,160) .. (340.3,160) .. controls (351.29,160) and (360.2,168.91) .. (360.2,179.9) .. controls (360.2,190.89) and (351.29,199.8) .. (340.3,199.8) .. controls (329.31,199.8) and (320.4,190.89) .. (320.4,179.9) -- cycle ;
\draw  [fill={rgb, 255:red, 255; green, 255; blue, 255 }  ,fill opacity=1 ] (440.4,180.3) .. controls (440.4,169.31) and (449.31,160.4) .. (460.3,160.4) .. controls (471.29,160.4) and (480.2,169.31) .. (480.2,180.3) .. controls (480.2,191.29) and (471.29,200.2) .. (460.3,200.2) .. controls (449.31,200.2) and (440.4,191.29) .. (440.4,180.3) -- cycle ;
\draw    (230,66) -- (139.27,100.1) ;
\draw [shift={(137.4,100.8)}, rotate = 339.4] [color={rgb, 255:red, 0; green, 0; blue, 0 }  ][line width=0.75]    (10.93,-3.29) .. controls (6.95,-1.4) and (3.31,-0.3) .. (0,0) .. controls (3.31,0.3) and (6.95,1.4) .. (10.93,3.29)   ;
\draw    (108.2,126.4) -- (70.89,159.86) ;
\draw [shift={(69.4,161.2)}, rotate = 318.11] [color={rgb, 255:red, 0; green, 0; blue, 0 }  ][line width=0.75]    (10.93,-3.29) .. controls (6.95,-1.4) and (3.31,-0.3) .. (0,0) .. controls (3.31,0.3) and (6.95,1.4) .. (10.93,3.29)   ;
\draw [color={rgb, 255:red, 144; green, 19; blue, 254 }  ,draw opacity=1 ] [dash pattern={on 0.84pt off 2.51pt}]  (52.2,200) -- (33.73,247.34) ;
\draw [shift={(33,249.2)}, rotate = 291.32] [color={rgb, 255:red, 144; green, 19; blue, 254 }  ,draw opacity=1 ][line width=0.75]    (10.93,-3.29) .. controls (6.95,-1.4) and (3.31,-0.3) .. (0,0) .. controls (3.31,0.3) and (6.95,1.4) .. (10.93,3.29)   ;
\draw [color={rgb, 255:red, 144; green, 19; blue, 254 }  ,draw opacity=1 ][line width=0.75]  [dash pattern={on 0.84pt off 2.51pt}]  (388.2,125.6) -- (350.89,159.06) ;
\draw [shift={(349.4,160.4)}, rotate = 318.11] [color={rgb, 255:red, 144; green, 19; blue, 254 }  ,draw opacity=1 ][line width=0.75]    (10.93,-3.29) .. controls (6.95,-1.4) and (3.31,-0.3) .. (0,0) .. controls (3.31,0.3) and (6.95,1.4) .. (10.93,3.29)   ;
\draw  [dash pattern={on 0.84pt off 2.51pt}]  (332.2,199.2) -- (313.73,246.54) ;
\draw [shift={(313,248.4)}, rotate = 291.32] [color={rgb, 255:red, 0; green, 0; blue, 0 }  ][line width=0.75]    (10.93,-3.29) .. controls (6.95,-1.4) and (3.31,-0.3) .. (0,0) .. controls (3.31,0.3) and (6.95,1.4) .. (10.93,3.29)   ;
\draw [color={rgb, 255:red, 144; green, 19; blue, 254 }  ,draw opacity=1 ] [dash pattern={on 0.84pt off 2.51pt}]  (135,125.6) -- (168.81,160.56) ;
\draw [shift={(170.2,162)}, rotate = 225.96] [color={rgb, 255:red, 144; green, 19; blue, 254 }  ,draw opacity=1 ][line width=0.75]    (10.93,-3.29) .. controls (6.95,-1.4) and (3.31,-0.3) .. (0,0) .. controls (3.31,0.3) and (6.95,1.4) .. (10.93,3.29)   ;
\draw [color={rgb, 255:red, 144; green, 19; blue, 254 }  ,draw opacity=1 ] [dash pattern={on 0.84pt off 2.51pt}]  (69,199.2) -- (85.93,246.92) ;
\draw [shift={(86.6,248.8)}, rotate = 250.46] [color={rgb, 255:red, 144; green, 19; blue, 254 }  ,draw opacity=1 ][line width=0.75]    (10.93,-3.29) .. controls (6.95,-1.4) and (3.31,-0.3) .. (0,0) .. controls (3.31,0.3) and (6.95,1.4) .. (10.93,3.29)   ;
\draw  [dash pattern={on 0.84pt off 2.51pt}]  (173,200) -- (154.53,247.34) ;
\draw [shift={(153.8,249.2)}, rotate = 291.32] [color={rgb, 255:red, 0; green, 0; blue, 0 }  ][line width=0.75]    (10.93,-3.29) .. controls (6.95,-1.4) and (3.31,-0.3) .. (0,0) .. controls (3.31,0.3) and (6.95,1.4) .. (10.93,3.29)   ;
\draw  [dash pattern={on 0.84pt off 2.51pt}]  (189.8,199.2) -- (206.73,246.92) ;
\draw [shift={(207.4,248.8)}, rotate = 250.46] [color={rgb, 255:red, 0; green, 0; blue, 0 }  ][line width=0.75]    (10.93,-3.29) .. controls (6.95,-1.4) and (3.31,-0.3) .. (0,0) .. controls (3.31,0.3) and (6.95,1.4) .. (10.93,3.29)   ;
\draw [color={rgb, 255:red, 0; green, 0; blue, 0 }  ,draw opacity=1 ][line width=0.75]  [dash pattern={on 0.84pt off 2.51pt}]  (348.2,198.4) -- (365.13,246.12) ;
\draw [shift={(365.8,248)}, rotate = 250.46] [color={rgb, 255:red, 0; green, 0; blue, 0 }  ,draw opacity=1 ][line width=0.75]    (10.93,-3.29) .. controls (6.95,-1.4) and (3.31,-0.3) .. (0,0) .. controls (3.31,0.3) and (6.95,1.4) .. (10.93,3.29)   ;
\draw  [dash pattern={on 0.84pt off 2.51pt}]  (452.2,199.2) -- (433.73,246.54) ;
\draw [shift={(433,248.4)}, rotate = 291.32] [color={rgb, 255:red, 0; green, 0; blue, 0 }  ][line width=0.75]    (10.93,-3.29) .. controls (6.95,-1.4) and (3.31,-0.3) .. (0,0) .. controls (3.31,0.3) and (6.95,1.4) .. (10.93,3.29)   ;
\draw  [dash pattern={on 0.84pt off 2.51pt}]  (469,198.4) -- (485.93,246.12) ;
\draw [shift={(486.6,248)}, rotate = 250.46] [color={rgb, 255:red, 0; green, 0; blue, 0 }  ][line width=0.75]    (10.93,-3.29) .. controls (6.95,-1.4) and (3.31,-0.3) .. (0,0) .. controls (3.31,0.3) and (6.95,1.4) .. (10.93,3.29)   ;
\draw [color={rgb, 255:red, 144; green, 19; blue, 254 }  ,draw opacity=1 ] [dash pattern={on 0.84pt off 2.51pt}]  (412.2,126.4) -- (446.01,161.36) ;
\draw [shift={(447.4,162.8)}, rotate = 225.96] [color={rgb, 255:red, 144; green, 19; blue, 254 }  ,draw opacity=1 ][line width=0.75]    (10.93,-3.29) .. controls (6.95,-1.4) and (3.31,-0.3) .. (0,0) .. controls (3.31,0.3) and (6.95,1.4) .. (10.93,3.29)   ;
\draw [color={rgb, 255:red, 0; green, 0; blue, 0 }  ,draw opacity=1 ][line width=0.75]    (269.6,65.2) -- (379.59,99.65) ;
\draw [shift={(381.5,100.25)}, rotate = 197.39] [color={rgb, 255:red, 0; green, 0; blue, 0 }  ,draw opacity=1 ][line width=0.75]    (10.93,-3.29) .. controls (6.95,-1.4) and (3.31,-0.3) .. (0,0) .. controls (3.31,0.3) and (6.95,1.4) .. (10.93,3.29)   ;

\draw (242,51.4) node [anchor=north west][inner sep=0.75pt]    {$S_{0}$};
\draw (391,101.4) node [anchor=north west][inner sep=0.75pt]    {$S_{1}$};
\draw (331,171.4) node [anchor=north west][inner sep=0.75pt]    {$S_{2}$};
\draw (451,171.4) node [anchor=north west][inner sep=0.75pt]    {$S_{3}$};
\draw (111.67,102.07) node [anchor=north west][inner sep=0.75pt]    {$S_{4}$};
\draw (52.67,171.4) node [anchor=north west][inner sep=0.75pt]    {$S_{5}$};
\draw (168.67,171.73) node [anchor=north west][inner sep=0.75pt]    {$S_{6}$};
\draw (362.37,200.07) node [anchor=north west][inner sep=0.75pt]  [font=\Large,color={rgb, 255:red, 74; green, 144; blue, 226 }  ,opacity=1 ]  {$z^{*}$};
\draw (437.03,120.07) node [anchor=north west][inner sep=0.75pt]  [font=\Large,color={rgb, 255:red, 0; green, 0; blue, 0 }  ,opacity=1 ]  {$z_{t}$};
\draw (350.7,117.4) node [anchor=north west][inner sep=0.75pt]  [font=\Large,color={rgb, 255:red, 74; green, 144; blue, 226 }  ,opacity=1 ]  {$z_{t}^{*}$};
\draw (31.33,42.73) node [anchor=north west][inner sep=0.75pt]  [font=\Huge,color={rgb, 255:red, 0; green, 0; blue, 0 }  ,opacity=1 ]  {$\mathcal{T}$};
\draw (313.63,56.07) node [anchor=north west][inner sep=0.75pt]  [font=\Large,color={rgb, 255:red, 0; green, 0; blue, 0 }  ,opacity=1 ]  {$A^{*}$};

\end{tikzpicture}

%% file: tikz/tree_search_0.tex
\tikzset{every picture/.style={line width=0.75pt}} 

\begin{tikzpicture}[x=0.75pt,y=0.75pt,yscale=-1,xscale=1]

\draw  [draw opacity=0][fill={rgb, 255:red, 232; green, 250; blue, 252 }  ,fill opacity=1 ] (0,30.25) -- (520.5,30.25) -- (520.5,140.75) -- (0,140.75) -- cycle ;
\draw   (100.8,109.9) .. controls (100.8,98.91) and (109.71,90) .. (120.7,90) .. controls (131.69,90) and (140.6,98.91) .. (140.6,109.9) .. controls (140.6,120.89) and (131.69,129.8) .. (120.7,129.8) .. controls (109.71,129.8) and (100.8,120.89) .. (100.8,109.9) -- cycle ;
\draw  [fill={rgb, 255:red, 208; green, 208; blue, 208 }  ,fill opacity=1 ] (230,59.9) .. controls (230,48.91) and (238.91,40) .. (249.9,40) .. controls (260.89,40) and (269.8,48.91) .. (269.8,59.9) .. controls (269.8,70.89) and (260.89,79.8) .. (249.9,79.8) .. controls (238.91,79.8) and (230,70.89) .. (230,59.9) -- cycle ;
\draw  [fill={rgb, 255:red, 208; green, 208; blue, 208 }  ,fill opacity=1 ] (380.4,109.9) .. controls (380.4,98.91) and (389.31,90) .. (400.3,90) .. controls (411.29,90) and (420.2,98.91) .. (420.2,109.9) .. controls (420.2,120.89) and (411.29,129.8) .. (400.3,129.8) .. controls (389.31,129.8) and (380.4,120.89) .. (380.4,109.9) -- cycle ;
\draw   (40.4,180.3) .. controls (40.4,169.31) and (49.31,160.4) .. (60.3,160.4) .. controls (71.29,160.4) and (80.2,169.31) .. (80.2,180.3) .. controls (80.2,191.29) and (71.29,200.2) .. (60.3,200.2) .. controls (49.31,200.2) and (40.4,191.29) .. (40.4,180.3) -- cycle ;
\draw   (160.4,180.3) .. controls (160.4,169.31) and (169.31,160.4) .. (180.3,160.4) .. controls (191.29,160.4) and (200.2,169.31) .. (200.2,180.3) .. controls (200.2,191.29) and (191.29,200.2) .. (180.3,200.2) .. controls (169.31,200.2) and (160.4,191.29) .. (160.4,180.3) -- cycle ;
\draw   (320.4,179.9) .. controls (320.4,168.91) and (329.31,160) .. (340.3,160) .. controls (351.29,160) and (360.2,168.91) .. (360.2,179.9) .. controls (360.2,190.89) and (351.29,199.8) .. (340.3,199.8) .. controls (329.31,199.8) and (320.4,190.89) .. (320.4,179.9) -- cycle ;
\draw   (440.4,180.3) .. controls (440.4,169.31) and (449.31,160.4) .. (460.3,160.4) .. controls (471.29,160.4) and (480.2,169.31) .. (480.2,180.3) .. controls (480.2,191.29) and (471.29,200.2) .. (460.3,200.2) .. controls (449.31,200.2) and (440.4,191.29) .. (440.4,180.3) -- cycle ;
\draw   (10,269.9) .. controls (10,258.91) and (18.91,250) .. (29.9,250) .. controls (40.89,250) and (49.8,258.91) .. (49.8,269.9) .. controls (49.8,280.89) and (40.89,289.8) .. (29.9,289.8) .. controls (18.91,289.8) and (10,280.89) .. (10,269.9) -- cycle ;
\draw   (70,270.1) .. controls (70,259.11) and (78.91,250.2) .. (89.9,250.2) .. controls (100.89,250.2) and (109.8,259.11) .. (109.8,270.1) .. controls (109.8,281.09) and (100.89,290) .. (89.9,290) .. controls (78.91,290) and (70,281.09) .. (70,270.1) -- cycle ;
\draw   (130,269.9) .. controls (130,258.91) and (138.91,250) .. (149.9,250) .. controls (160.89,250) and (169.8,258.91) .. (169.8,269.9) .. controls (169.8,280.89) and (160.89,289.8) .. (149.9,289.8) .. controls (138.91,289.8) and (130,280.89) .. (130,269.9) -- cycle ;
\draw   (190,270.1) .. controls (190,259.11) and (198.91,250.2) .. (209.9,250.2) .. controls (220.89,250.2) and (229.8,259.11) .. (229.8,270.1) .. controls (229.8,281.09) and (220.89,290) .. (209.9,290) .. controls (198.91,290) and (190,281.09) .. (190,270.1) -- cycle ;
\draw   (290,270.3) .. controls (290,259.31) and (298.91,250.4) .. (309.9,250.4) .. controls (320.89,250.4) and (329.8,259.31) .. (329.8,270.3) .. controls (329.8,281.29) and (320.89,290.2) .. (309.9,290.2) .. controls (298.91,290.2) and (290,281.29) .. (290,270.3) -- cycle ;
\draw   (350,270.5) .. controls (350,259.51) and (358.91,250.6) .. (369.9,250.6) .. controls (380.89,250.6) and (389.8,259.51) .. (389.8,270.5) .. controls (389.8,281.49) and (380.89,290.4) .. (369.9,290.4) .. controls (358.91,290.4) and (350,281.49) .. (350,270.5) -- cycle ;
\draw   (410,270.3) .. controls (410,259.31) and (418.91,250.4) .. (429.9,250.4) .. controls (440.89,250.4) and (449.8,259.31) .. (449.8,270.3) .. controls (449.8,281.29) and (440.89,290.2) .. (429.9,290.2) .. controls (418.91,290.2) and (410,281.29) .. (410,270.3) -- cycle ;
\draw   (470,270.5) .. controls (470,259.51) and (478.91,250.6) .. (489.9,250.6) .. controls (500.89,250.6) and (509.8,259.51) .. (509.8,270.5) .. controls (509.8,281.49) and (500.89,290.4) .. (489.9,290.4) .. controls (478.91,290.4) and (470,281.49) .. (470,270.5) -- cycle ;
\draw    (230,66) -- (139.27,100.1) ;
\draw [shift={(137.4,100.8)}, rotate = 339.4] [color={rgb, 255:red, 0; green, 0; blue, 0 }  ][line width=0.75]    (10.93,-3.29) .. controls (6.95,-1.4) and (3.31,-0.3) .. (0,0) .. controls (3.31,0.3) and (6.95,1.4) .. (10.93,3.29)   ;
\draw    (108.2,126.4) -- (70.89,159.86) ;
\draw [shift={(69.4,161.2)}, rotate = 318.11] [color={rgb, 255:red, 0; green, 0; blue, 0 }  ][line width=0.75]    (10.93,-3.29) .. controls (6.95,-1.4) and (3.31,-0.3) .. (0,0) .. controls (3.31,0.3) and (6.95,1.4) .. (10.93,3.29)   ;
\draw    (52.2,200) -- (33.73,247.34) ;
\draw [shift={(33,249.2)}, rotate = 291.32] [color={rgb, 255:red, 0; green, 0; blue, 0 }  ][line width=0.75]    (10.93,-3.29) .. controls (6.95,-1.4) and (3.31,-0.3) .. (0,0) .. controls (3.31,0.3) and (6.95,1.4) .. (10.93,3.29)   ;
\draw    (388.2,125.6) -- (350.89,159.06) ;
\draw [shift={(349.4,160.4)}, rotate = 318.11] [color={rgb, 255:red, 0; green, 0; blue, 0 }  ][line width=0.75]    (10.93,-3.29) .. controls (6.95,-1.4) and (3.31,-0.3) .. (0,0) .. controls (3.31,0.3) and (6.95,1.4) .. (10.93,3.29)   ;
\draw    (332.2,199.2) -- (313.73,246.54) ;
\draw [shift={(313,248.4)}, rotate = 291.32] [color={rgb, 255:red, 0; green, 0; blue, 0 }  ][line width=0.75]    (10.93,-3.29) .. controls (6.95,-1.4) and (3.31,-0.3) .. (0,0) .. controls (3.31,0.3) and (6.95,1.4) .. (10.93,3.29)   ;
\draw    (135,125.6) -- (168.81,160.56) ;
\draw [shift={(170.2,162)}, rotate = 225.96] [color={rgb, 255:red, 0; green, 0; blue, 0 }  ][line width=0.75]    (10.93,-3.29) .. controls (6.95,-1.4) and (3.31,-0.3) .. (0,0) .. controls (3.31,0.3) and (6.95,1.4) .. (10.93,3.29)   ;
\draw    (69,199.2) -- (85.93,246.92) ;
\draw [shift={(86.6,248.8)}, rotate = 250.46] [color={rgb, 255:red, 0; green, 0; blue, 0 }  ][line width=0.75]    (10.93,-3.29) .. controls (6.95,-1.4) and (3.31,-0.3) .. (0,0) .. controls (3.31,0.3) and (6.95,1.4) .. (10.93,3.29)   ;
\draw    (173,200) -- (154.53,247.34) ;
\draw [shift={(153.8,249.2)}, rotate = 291.32] [color={rgb, 255:red, 0; green, 0; blue, 0 }  ][line width=0.75]    (10.93,-3.29) .. controls (6.95,-1.4) and (3.31,-0.3) .. (0,0) .. controls (3.31,0.3) and (6.95,1.4) .. (10.93,3.29)   ;
\draw    (189.8,199.2) -- (206.73,246.92) ;
\draw [shift={(207.4,248.8)}, rotate = 250.46] [color={rgb, 255:red, 0; green, 0; blue, 0 }  ][line width=0.75]    (10.93,-3.29) .. controls (6.95,-1.4) and (3.31,-0.3) .. (0,0) .. controls (3.31,0.3) and (6.95,1.4) .. (10.93,3.29)   ;
\draw    (348.2,198.4) -- (365.13,246.12) ;
\draw [shift={(365.8,248)}, rotate = 250.46] [color={rgb, 255:red, 0; green, 0; blue, 0 }  ][line width=0.75]    (10.93,-3.29) .. controls (6.95,-1.4) and (3.31,-0.3) .. (0,0) .. controls (3.31,0.3) and (6.95,1.4) .. (10.93,3.29)   ;
\draw    (452.2,199.2) -- (433.73,246.54) ;
\draw [shift={(433,248.4)}, rotate = 291.32] [color={rgb, 255:red, 0; green, 0; blue, 0 }  ][line width=0.75]    (10.93,-3.29) .. controls (6.95,-1.4) and (3.31,-0.3) .. (0,0) .. controls (3.31,0.3) and (6.95,1.4) .. (10.93,3.29)   ;
\draw    (469,198.4) -- (485.93,246.12) ;
\draw [shift={(486.6,248)}, rotate = 250.46] [color={rgb, 255:red, 0; green, 0; blue, 0 }  ][line width=0.75]    (10.93,-3.29) .. controls (6.95,-1.4) and (3.31,-0.3) .. (0,0) .. controls (3.31,0.3) and (6.95,1.4) .. (10.93,3.29)   ;
\draw    (412.2,126.4) -- (446.01,161.36) ;
\draw [shift={(447.4,162.8)}, rotate = 225.96] [color={rgb, 255:red, 0; green, 0; blue, 0 }  ][line width=0.75]    (10.93,-3.29) .. controls (6.95,-1.4) and (3.31,-0.3) .. (0,0) .. controls (3.31,0.3) and (6.95,1.4) .. (10.93,3.29)   ;
\draw    (269.6,65.2) -- (379.59,99.65) ;
\draw [shift={(381.5,100.25)}, rotate = 197.39] [color={rgb, 255:red, 0; green, 0; blue, 0 }  ][line width=0.75]    (10.93,-3.29) .. controls (6.95,-1.4) and (3.31,-0.3) .. (0,0) .. controls (3.31,0.3) and (6.95,1.4) .. (10.93,3.29)   ;

\draw (240,48) node [anchor=north west][inner sep=0.75pt]    {$S_{0}$};
\draw (50.5,61.5) node [anchor=north west][inner sep=0.75pt]   [align=left] {Sample $\displaystyle Q( s_{0} ,a_{1})$};
\draw (344,62) node [anchor=north west][inner sep=0.75pt]   [align=left] {Sample $\displaystyle Q( s_{0} ,a_{2})$};
\draw (388,98) node [anchor=north west][inner sep=0.75pt]    {$S_{1}$};

\end{tikzpicture}

%% file: tikz/tree_search_1.tex
\tikzset{every picture/.style={line width=0.75pt}} 

\begin{tikzpicture}[x=0.75pt,y=0.75pt,yscale=-1,xscale=1]

\draw  [draw opacity=0][fill={rgb, 255:red, 232; green, 250; blue, 252 }  ,fill opacity=1 ] (237,86.25) -- (522,86.25) -- (522,219) -- (237,219) -- cycle ;
\draw   (100.8,109.9) .. controls (100.8,98.91) and (109.71,90) .. (120.7,90) .. controls (131.69,90) and (140.6,98.91) .. (140.6,109.9) .. controls (140.6,120.89) and (131.69,129.8) .. (120.7,129.8) .. controls (109.71,129.8) and (100.8,120.89) .. (100.8,109.9) -- cycle ;
\draw  [fill={rgb, 255:red, 208; green, 208; blue, 208 }  ,fill opacity=1 ] (230,59.9) .. controls (230,48.91) and (238.91,40) .. (249.9,40) .. controls (260.89,40) and (269.8,48.91) .. (269.8,59.9) .. controls (269.8,70.89) and (260.89,79.8) .. (249.9,79.8) .. controls (238.91,79.8) and (230,70.89) .. (230,59.9) -- cycle ;
\draw  [fill={rgb, 255:red, 208; green, 208; blue, 208 }  ,fill opacity=1 ] (380.4,109.9) .. controls (380.4,98.91) and (389.31,90) .. (400.3,90) .. controls (411.29,90) and (420.2,98.91) .. (420.2,109.9) .. controls (420.2,120.89) and (411.29,129.8) .. (400.3,129.8) .. controls (389.31,129.8) and (380.4,120.89) .. (380.4,109.9) -- cycle ;
\draw   (40.4,180.3) .. controls (40.4,169.31) and (49.31,160.4) .. (60.3,160.4) .. controls (71.29,160.4) and (80.2,169.31) .. (80.2,180.3) .. controls (80.2,191.29) and (71.29,200.2) .. (60.3,200.2) .. controls (49.31,200.2) and (40.4,191.29) .. (40.4,180.3) -- cycle ;
\draw   (160.4,180.3) .. controls (160.4,169.31) and (169.31,160.4) .. (180.3,160.4) .. controls (191.29,160.4) and (200.2,169.31) .. (200.2,180.3) .. controls (200.2,191.29) and (191.29,200.2) .. (180.3,200.2) .. controls (169.31,200.2) and (160.4,191.29) .. (160.4,180.3) -- cycle ;
\draw   (320.4,179.9) .. controls (320.4,168.91) and (329.31,160) .. (340.3,160) .. controls (351.29,160) and (360.2,168.91) .. (360.2,179.9) .. controls (360.2,190.89) and (351.29,199.8) .. (340.3,199.8) .. controls (329.31,199.8) and (320.4,190.89) .. (320.4,179.9) -- cycle ;
\draw   (440.4,180.3) .. controls (440.4,169.31) and (449.31,160.4) .. (460.3,160.4) .. controls (471.29,160.4) and (480.2,169.31) .. (480.2,180.3) .. controls (480.2,191.29) and (471.29,200.2) .. (460.3,200.2) .. controls (449.31,200.2) and (440.4,191.29) .. (440.4,180.3) -- cycle ;
\draw   (10,269.9) .. controls (10,258.91) and (18.91,250) .. (29.9,250) .. controls (40.89,250) and (49.8,258.91) .. (49.8,269.9) .. controls (49.8,280.89) and (40.89,289.8) .. (29.9,289.8) .. controls (18.91,289.8) and (10,280.89) .. (10,269.9) -- cycle ;
\draw   (70,270.1) .. controls (70,259.11) and (78.91,250.2) .. (89.9,250.2) .. controls (100.89,250.2) and (109.8,259.11) .. (109.8,270.1) .. controls (109.8,281.09) and (100.89,290) .. (89.9,290) .. controls (78.91,290) and (70,281.09) .. (70,270.1) -- cycle ;
\draw   (130,269.9) .. controls (130,258.91) and (138.91,250) .. (149.9,250) .. controls (160.89,250) and (169.8,258.91) .. (169.8,269.9) .. controls (169.8,280.89) and (160.89,289.8) .. (149.9,289.8) .. controls (138.91,289.8) and (130,280.89) .. (130,269.9) -- cycle ;
\draw   (190,270.1) .. controls (190,259.11) and (198.91,250.2) .. (209.9,250.2) .. controls (220.89,250.2) and (229.8,259.11) .. (229.8,270.1) .. controls (229.8,281.09) and (220.89,290) .. (209.9,290) .. controls (198.91,290) and (190,281.09) .. (190,270.1) -- cycle ;
\draw   (290,270.3) .. controls (290,259.31) and (298.91,250.4) .. (309.9,250.4) .. controls (320.89,250.4) and (329.8,259.31) .. (329.8,270.3) .. controls (329.8,281.29) and (320.89,290.2) .. (309.9,290.2) .. controls (298.91,290.2) and (290,281.29) .. (290,270.3) -- cycle ;
\draw   (350,270.5) .. controls (350,259.51) and (358.91,250.6) .. (369.9,250.6) .. controls (380.89,250.6) and (389.8,259.51) .. (389.8,270.5) .. controls (389.8,281.49) and (380.89,290.4) .. (369.9,290.4) .. controls (358.91,290.4) and (350,281.49) .. (350,270.5) -- cycle ;
\draw   (410,270.3) .. controls (410,259.31) and (418.91,250.4) .. (429.9,250.4) .. controls (440.89,250.4) and (449.8,259.31) .. (449.8,270.3) .. controls (449.8,281.29) and (440.89,290.2) .. (429.9,290.2) .. controls (418.91,290.2) and (410,281.29) .. (410,270.3) -- cycle ;
\draw   (470,270.5) .. controls (470,259.51) and (478.91,250.6) .. (489.9,250.6) .. controls (500.89,250.6) and (509.8,259.51) .. (509.8,270.5) .. controls (509.8,281.49) and (500.89,290.4) .. (489.9,290.4) .. controls (478.91,290.4) and (470,281.49) .. (470,270.5) -- cycle ;
\draw    (230,66) -- (139.27,100.1) ;
\draw [shift={(137.4,100.8)}, rotate = 339.4] [color={rgb, 255:red, 0; green, 0; blue, 0 }  ][line width=0.75]    (10.93,-3.29) .. controls (6.95,-1.4) and (3.31,-0.3) .. (0,0) .. controls (3.31,0.3) and (6.95,1.4) .. (10.93,3.29)   ;
\draw    (108.2,126.4) -- (70.89,159.86) ;
\draw [shift={(69.4,161.2)}, rotate = 318.11] [color={rgb, 255:red, 0; green, 0; blue, 0 }  ][line width=0.75]    (10.93,-3.29) .. controls (6.95,-1.4) and (3.31,-0.3) .. (0,0) .. controls (3.31,0.3) and (6.95,1.4) .. (10.93,3.29)   ;
\draw    (52.2,200) -- (33.73,247.34) ;
\draw [shift={(33,249.2)}, rotate = 291.32] [color={rgb, 255:red, 0; green, 0; blue, 0 }  ][line width=0.75]    (10.93,-3.29) .. controls (6.95,-1.4) and (3.31,-0.3) .. (0,0) .. controls (3.31,0.3) and (6.95,1.4) .. (10.93,3.29)   ;
\draw    (388.2,125.6) -- (350.89,159.06) ;
\draw [shift={(349.4,160.4)}, rotate = 318.11] [color={rgb, 255:red, 0; green, 0; blue, 0 }  ][line width=0.75]    (10.93,-3.29) .. controls (6.95,-1.4) and (3.31,-0.3) .. (0,0) .. controls (3.31,0.3) and (6.95,1.4) .. (10.93,3.29)   ;
\draw    (332.2,199.2) -- (313.73,246.54) ;
\draw [shift={(313,248.4)}, rotate = 291.32] [color={rgb, 255:red, 0; green, 0; blue, 0 }  ][line width=0.75]    (10.93,-3.29) .. controls (6.95,-1.4) and (3.31,-0.3) .. (0,0) .. controls (3.31,0.3) and (6.95,1.4) .. (10.93,3.29)   ;
\draw    (135,125.6) -- (168.81,160.56) ;
\draw [shift={(170.2,162)}, rotate = 225.96] [color={rgb, 255:red, 0; green, 0; blue, 0 }  ][line width=0.75]    (10.93,-3.29) .. controls (6.95,-1.4) and (3.31,-0.3) .. (0,0) .. controls (3.31,0.3) and (6.95,1.4) .. (10.93,3.29)   ;
\draw    (69,199.2) -- (85.93,246.92) ;
\draw [shift={(86.6,248.8)}, rotate = 250.46] [color={rgb, 255:red, 0; green, 0; blue, 0 }  ][line width=0.75]    (10.93,-3.29) .. controls (6.95,-1.4) and (3.31,-0.3) .. (0,0) .. controls (3.31,0.3) and (6.95,1.4) .. (10.93,3.29)   ;
\draw    (173,200) -- (154.53,247.34) ;
\draw [shift={(153.8,249.2)}, rotate = 291.32] [color={rgb, 255:red, 0; green, 0; blue, 0 }  ][line width=0.75]    (10.93,-3.29) .. controls (6.95,-1.4) and (3.31,-0.3) .. (0,0) .. controls (3.31,0.3) and (6.95,1.4) .. (10.93,3.29)   ;
\draw    (189.8,199.2) -- (206.73,246.92) ;
\draw [shift={(207.4,248.8)}, rotate = 250.46] [color={rgb, 255:red, 0; green, 0; blue, 0 }  ][line width=0.75]    (10.93,-3.29) .. controls (6.95,-1.4) and (3.31,-0.3) .. (0,0) .. controls (3.31,0.3) and (6.95,1.4) .. (10.93,3.29)   ;
\draw    (348.2,198.4) -- (365.13,246.12) ;
\draw [shift={(365.8,248)}, rotate = 250.46] [color={rgb, 255:red, 0; green, 0; blue, 0 }  ][line width=0.75]    (10.93,-3.29) .. controls (6.95,-1.4) and (3.31,-0.3) .. (0,0) .. controls (3.31,0.3) and (6.95,1.4) .. (10.93,3.29)   ;
\draw    (452.2,199.2) -- (433.73,246.54) ;
\draw [shift={(433,248.4)}, rotate = 291.32] [color={rgb, 255:red, 0; green, 0; blue, 0 }  ][line width=0.75]    (10.93,-3.29) .. controls (6.95,-1.4) and (3.31,-0.3) .. (0,0) .. controls (3.31,0.3) and (6.95,1.4) .. (10.93,3.29)   ;
\draw    (469,198.4) -- (485.93,246.12) ;
\draw [shift={(486.6,248)}, rotate = 250.46] [color={rgb, 255:red, 0; green, 0; blue, 0 }  ][line width=0.75]    (10.93,-3.29) .. controls (6.95,-1.4) and (3.31,-0.3) .. (0,0) .. controls (3.31,0.3) and (6.95,1.4) .. (10.93,3.29)   ;
\draw    (412.2,126.4) -- (446.01,161.36) ;
\draw [shift={(447.4,162.8)}, rotate = 225.96] [color={rgb, 255:red, 0; green, 0; blue, 0 }  ][line width=0.75]    (10.93,-3.29) .. controls (6.95,-1.4) and (3.31,-0.3) .. (0,0) .. controls (3.31,0.3) and (6.95,1.4) .. (10.93,3.29)   ;
\draw    (269.6,65.2) -- (379.59,99.65) ;
\draw [shift={(381.5,100.25)}, rotate = 197.39] [color={rgb, 255:red, 0; green, 0; blue, 0 }  ][line width=0.75]    (10.93,-3.29) .. controls (6.95,-1.4) and (3.31,-0.3) .. (0,0) .. controls (3.31,0.3) and (6.95,1.4) .. (10.93,3.29)   ;

\draw (240,48) node [anchor=north west][inner sep=0.75pt]    {$S_{0}$};
\draw (248.5,158.5) node [anchor=north west][inner sep=0.75pt]   [align=left] {Sample\\$\displaystyle Q( s_{1} ,a_{1})$};
\draw (373,158) node [anchor=north west][inner sep=0.75pt]   [align=left] {Sample\\$\displaystyle Q( s_{1} ,a_{2})$};
\draw (389,97) node [anchor=north west][inner sep=0.75pt]    {$S_{1}$};

\end{tikzpicture}

%% file: tikz/tree_search_2.tex
\tikzset{every picture/.style={line width=0.75pt}} 

\begin{tikzpicture}[x=0.75pt,y=0.75pt,yscale=-1,xscale=1]

\draw   (100.8,109.9) .. controls (100.8,98.91) and (109.71,90) .. (120.7,90) .. controls (131.69,90) and (140.6,98.91) .. (140.6,109.9) .. controls (140.6,120.89) and (131.69,129.8) .. (120.7,129.8) .. controls (109.71,129.8) and (100.8,120.89) .. (100.8,109.9) -- cycle ;
\draw  [fill={rgb, 255:red, 208; green, 208; blue, 208 }  ,fill opacity=1 ] (230,59.9) .. controls (230,48.91) and (238.91,40) .. (249.9,40) .. controls (260.89,40) and (269.8,48.91) .. (269.8,59.9) .. controls (269.8,70.89) and (260.89,79.8) .. (249.9,79.8) .. controls (238.91,79.8) and (230,70.89) .. (230,59.9) -- cycle ;
\draw  [fill={rgb, 255:red, 208; green, 208; blue, 208 }  ,fill opacity=1 ] (380.4,109.9) .. controls (380.4,98.91) and (389.31,90) .. (400.3,90) .. controls (411.29,90) and (420.2,98.91) .. (420.2,109.9) .. controls (420.2,120.89) and (411.29,129.8) .. (400.3,129.8) .. controls (389.31,129.8) and (380.4,120.89) .. (380.4,109.9) -- cycle ;
\draw   (40.4,180.3) .. controls (40.4,169.31) and (49.31,160.4) .. (60.3,160.4) .. controls (71.29,160.4) and (80.2,169.31) .. (80.2,180.3) .. controls (80.2,191.29) and (71.29,200.2) .. (60.3,200.2) .. controls (49.31,200.2) and (40.4,191.29) .. (40.4,180.3) -- cycle ;
\draw   (160.4,180.3) .. controls (160.4,169.31) and (169.31,160.4) .. (180.3,160.4) .. controls (191.29,160.4) and (200.2,169.31) .. (200.2,180.3) .. controls (200.2,191.29) and (191.29,200.2) .. (180.3,200.2) .. controls (169.31,200.2) and (160.4,191.29) .. (160.4,180.3) -- cycle ;
\draw  [fill={rgb, 255:red, 208; green, 208; blue, 208 }  ,fill opacity=1 ] (320.4,179.9) .. controls (320.4,168.91) and (329.31,160) .. (340.3,160) .. controls (351.29,160) and (360.2,168.91) .. (360.2,179.9) .. controls (360.2,190.89) and (351.29,199.8) .. (340.3,199.8) .. controls (329.31,199.8) and (320.4,190.89) .. (320.4,179.9) -- cycle ;
\draw   (440.4,180.3) .. controls (440.4,169.31) and (449.31,160.4) .. (460.3,160.4) .. controls (471.29,160.4) and (480.2,169.31) .. (480.2,180.3) .. controls (480.2,191.29) and (471.29,200.2) .. (460.3,200.2) .. controls (449.31,200.2) and (440.4,191.29) .. (440.4,180.3) -- cycle ;
\draw   (10,269.9) .. controls (10,258.91) and (18.91,250) .. (29.9,250) .. controls (40.89,250) and (49.8,258.91) .. (49.8,269.9) .. controls (49.8,280.89) and (40.89,289.8) .. (29.9,289.8) .. controls (18.91,289.8) and (10,280.89) .. (10,269.9) -- cycle ;
\draw   (70,270.1) .. controls (70,259.11) and (78.91,250.2) .. (89.9,250.2) .. controls (100.89,250.2) and (109.8,259.11) .. (109.8,270.1) .. controls (109.8,281.09) and (100.89,290) .. (89.9,290) .. controls (78.91,290) and (70,281.09) .. (70,270.1) -- cycle ;
\draw   (130,269.9) .. controls (130,258.91) and (138.91,250) .. (149.9,250) .. controls (160.89,250) and (169.8,258.91) .. (169.8,269.9) .. controls (169.8,280.89) and (160.89,289.8) .. (149.9,289.8) .. controls (138.91,289.8) and (130,280.89) .. (130,269.9) -- cycle ;
\draw   (190,270.1) .. controls (190,259.11) and (198.91,250.2) .. (209.9,250.2) .. controls (220.89,250.2) and (229.8,259.11) .. (229.8,270.1) .. controls (229.8,281.09) and (220.89,290) .. (209.9,290) .. controls (198.91,290) and (190,281.09) .. (190,270.1) -- cycle ;
\draw   (290,270.3) .. controls (290,259.31) and (298.91,250.4) .. (309.9,250.4) .. controls (320.89,250.4) and (329.8,259.31) .. (329.8,270.3) .. controls (329.8,281.29) and (320.89,290.2) .. (309.9,290.2) .. controls (298.91,290.2) and (290,281.29) .. (290,270.3) -- cycle ;
\draw   (350,270.5) .. controls (350,259.51) and (358.91,250.6) .. (369.9,250.6) .. controls (380.89,250.6) and (389.8,259.51) .. (389.8,270.5) .. controls (389.8,281.49) and (380.89,290.4) .. (369.9,290.4) .. controls (358.91,290.4) and (350,281.49) .. (350,270.5) -- cycle ;
\draw   (410,270.3) .. controls (410,259.31) and (418.91,250.4) .. (429.9,250.4) .. controls (440.89,250.4) and (449.8,259.31) .. (449.8,270.3) .. controls (449.8,281.29) and (440.89,290.2) .. (429.9,290.2) .. controls (418.91,290.2) and (410,281.29) .. (410,270.3) -- cycle ;
\draw   (470,270.5) .. controls (470,259.51) and (478.91,250.6) .. (489.9,250.6) .. controls (500.89,250.6) and (509.8,259.51) .. (509.8,270.5) .. controls (509.8,281.49) and (500.89,290.4) .. (489.9,290.4) .. controls (478.91,290.4) and (470,281.49) .. (470,270.5) -- cycle ;
\draw    (230,66) -- (139.27,100.1) ;
\draw [shift={(137.4,100.8)}, rotate = 339.4] [color={rgb, 255:red, 0; green, 0; blue, 0 }  ][line width=0.75]    (10.93,-3.29) .. controls (6.95,-1.4) and (3.31,-0.3) .. (0,0) .. controls (3.31,0.3) and (6.95,1.4) .. (10.93,3.29)   ;
\draw    (108.2,126.4) -- (70.89,159.86) ;
\draw [shift={(69.4,161.2)}, rotate = 318.11] [color={rgb, 255:red, 0; green, 0; blue, 0 }  ][line width=0.75]    (10.93,-3.29) .. controls (6.95,-1.4) and (3.31,-0.3) .. (0,0) .. controls (3.31,0.3) and (6.95,1.4) .. (10.93,3.29)   ;
\draw    (52.2,200) -- (33.73,247.34) ;
\draw [shift={(33,249.2)}, rotate = 291.32] [color={rgb, 255:red, 0; green, 0; blue, 0 }  ][line width=0.75]    (10.93,-3.29) .. controls (6.95,-1.4) and (3.31,-0.3) .. (0,0) .. controls (3.31,0.3) and (6.95,1.4) .. (10.93,3.29)   ;
\draw    (388.2,125.6) -- (350.89,159.06) ;
\draw [shift={(349.4,160.4)}, rotate = 318.11] [color={rgb, 255:red, 0; green, 0; blue, 0 }  ][line width=0.75]    (10.93,-3.29) .. controls (6.95,-1.4) and (3.31,-0.3) .. (0,0) .. controls (3.31,0.3) and (6.95,1.4) .. (10.93,3.29)   ;
\draw    (332.2,199.2) -- (313.73,246.54) ;
\draw [shift={(313,248.4)}, rotate = 291.32] [color={rgb, 255:red, 0; green, 0; blue, 0 }  ][line width=0.75]    (10.93,-3.29) .. controls (6.95,-1.4) and (3.31,-0.3) .. (0,0) .. controls (3.31,0.3) and (6.95,1.4) .. (10.93,3.29)   ;
\draw    (135,125.6) -- (168.81,160.56) ;
\draw [shift={(170.2,162)}, rotate = 225.96] [color={rgb, 255:red, 0; green, 0; blue, 0 }  ][line width=0.75]    (10.93,-3.29) .. controls (6.95,-1.4) and (3.31,-0.3) .. (0,0) .. controls (3.31,0.3) and (6.95,1.4) .. (10.93,3.29)   ;
\draw    (69,199.2) -- (85.93,246.92) ;
\draw [shift={(86.6,248.8)}, rotate = 250.46] [color={rgb, 255:red, 0; green, 0; blue, 0 }  ][line width=0.75]    (10.93,-3.29) .. controls (6.95,-1.4) and (3.31,-0.3) .. (0,0) .. controls (3.31,0.3) and (6.95,1.4) .. (10.93,3.29)   ;
\draw    (173,200) -- (154.53,247.34) ;
\draw [shift={(153.8,249.2)}, rotate = 291.32] [color={rgb, 255:red, 0; green, 0; blue, 0 }  ][line width=0.75]    (10.93,-3.29) .. controls (6.95,-1.4) and (3.31,-0.3) .. (0,0) .. controls (3.31,0.3) and (6.95,1.4) .. (10.93,3.29)   ;
\draw    (189.8,199.2) -- (206.73,246.92) ;
\draw [shift={(207.4,248.8)}, rotate = 250.46] [color={rgb, 255:red, 0; green, 0; blue, 0 }  ][line width=0.75]    (10.93,-3.29) .. controls (6.95,-1.4) and (3.31,-0.3) .. (0,0) .. controls (3.31,0.3) and (6.95,1.4) .. (10.93,3.29)   ;
\draw    (348.2,198.4) -- (365.13,246.12) ;
\draw [shift={(365.8,248)}, rotate = 250.46] [color={rgb, 255:red, 0; green, 0; blue, 0 }  ][line width=0.75]    (10.93,-3.29) .. controls (6.95,-1.4) and (3.31,-0.3) .. (0,0) .. controls (3.31,0.3) and (6.95,1.4) .. (10.93,3.29)   ;
\draw    (452.2,199.2) -- (433.73,246.54) ;
\draw [shift={(433,248.4)}, rotate = 291.32] [color={rgb, 255:red, 0; green, 0; blue, 0 }  ][line width=0.75]    (10.93,-3.29) .. controls (6.95,-1.4) and (3.31,-0.3) .. (0,0) .. controls (3.31,0.3) and (6.95,1.4) .. (10.93,3.29)   ;
\draw    (469,198.4) -- (485.93,246.12) ;
\draw [shift={(486.6,248)}, rotate = 250.46] [color={rgb, 255:red, 0; green, 0; blue, 0 }  ][line width=0.75]    (10.93,-3.29) .. controls (6.95,-1.4) and (3.31,-0.3) .. (0,0) .. controls (3.31,0.3) and (6.95,1.4) .. (10.93,3.29)   ;
\draw    (412.2,126.4) -- (446.01,161.36) ;
\draw [shift={(447.4,162.8)}, rotate = 225.96] [color={rgb, 255:red, 0; green, 0; blue, 0 }  ][line width=0.75]    (10.93,-3.29) .. controls (6.95,-1.4) and (3.31,-0.3) .. (0,0) .. controls (3.31,0.3) and (6.95,1.4) .. (10.93,3.29)   ;
\draw    (269.6,65.2) -- (379.59,99.65) ;
\draw [shift={(381.5,100.25)}, rotate = 197.39] [color={rgb, 255:red, 0; green, 0; blue, 0 }  ][line width=0.75]    (10.93,-3.29) .. controls (6.95,-1.4) and (3.31,-0.3) .. (0,0) .. controls (3.31,0.3) and (6.95,1.4) .. (10.93,3.29)   ;
\draw [color={rgb, 255:red, 144; green, 19; blue, 254 }  ,draw opacity=1 ][line width=2.25]    (362,166) -- (391.13,137.78) ;
\draw [shift={(394,135)}, rotate = 135.91] [color={rgb, 255:red, 144; green, 19; blue, 254 }  ,draw opacity=1 ][line width=2.25]    (17.49,-5.26) .. controls (11.12,-2.23) and (5.29,-0.48) .. (0,0) .. controls (5.29,0.48) and (11.12,2.23) .. (17.49,5.26)   ;
\draw [color={rgb, 255:red, 144; green, 19; blue, 254 }  ,draw opacity=1 ][line width=2.25]    (380.4,109.9) -- (271.82,76.19) ;
\draw [shift={(268,75)}, rotate = 17.25] [color={rgb, 255:red, 144; green, 19; blue, 254 }  ,draw opacity=1 ][line width=2.25]    (17.49,-5.26) .. controls (11.12,-2.23) and (5.29,-0.48) .. (0,0) .. controls (5.29,0.48) and (11.12,2.23) .. (17.49,5.26)   ;

\draw (240,48) node [anchor=north west][inner sep=0.75pt]    {$S_{0}$};
\draw (234.5,195.5) node [anchor=north west][inner sep=0.75pt]   [align=left] {$\displaystyle P( Q( s_{2} ,a_{1}))$};
\draw (355,197) node [anchor=north west][inner sep=0.75pt]   [align=left] {$\displaystyle P( Q( s_{2} ,a_{2}))$};
\draw (389,98) node [anchor=north west][inner sep=0.75pt]    {$S_{1}$};
\draw (329,168) node [anchor=north west][inner sep=0.75pt]    {$S_{2}$};
\draw (319.7,56.8) node [anchor=north west][inner sep=0.75pt]    {$r( s_{0} ,a_{2})$};
\draw (309.7,121.8) node [anchor=north west][inner sep=0.75pt]    {$r( s_{1} ,a_{1})$};

\end{tikzpicture}

%% file: proof.tex
For our analysis, we will consider an equivalent regret definition, based on a random variable sequence $Y_t$. Recall that at time $t$, the observation is $\obs_t$, and the information available \textit{after} observing $\obs_t$ is $\hist_{t+1} = \left\{\hist_{t}, \leaf_t,r(\leaf_t)\right\}$.

Let $Y_t \in \mathbb{R}^{|\leaves_t|}$ be a random variable that is distributed as follows. For each $\leaf \in \leaves_t$, let $\left\{\hist_{t}, \leaf, r(\leaf)\right\}$ denote the information set of the history $\hist_{t}$, and the observation that results from selecting $\leaf$ at time $t$. We let $P\left( \left. Y_t(\leaf) \right| \hist_{t} \right) = P\left( \left. Q_0(s_0,\roota(\leaf)) \right| \hist_{t}, \leaf, r(\leaf) \right)$, that is, for each possible leaf $\leaf$, $Y_t(\leaf)$ is drawn from the posterior reward \textit{after} observing that leaf.
We furthermore assume that $Y_t(\leaf) = f(\hist_{t},\leaf, r(\leaf), \mathcal{N}_t)$, where $f$ is some deterministic function, and $\mathcal{N}_t$ is a noise sequence that is independent of the past, and of $A^*$. This is a technical assumption that essentially states that sampling from the posterior is independent of the decision process. Note that in general, $Y_t$ is not necessarily i.i.d., and may depend on the action selection policy. The next proposition shows that we can define the Bayesian regret using $Y_t$.
\begin{proposition}\label{prop:reg_equiv}
We have that $\mathbb{E}\left[\textrm{Regret}(T)\right] = \mathbb{E}\left[\sum_{t=1}^T \left[Y_t(\leaf^*_t) - Y_t(\leaf_t)\right]\right]$.
\end{proposition}
\begin{proof}
Using the tower rule:

\begin{equation*}
\begin{split}
    \mathbb{E}\left[\sum_{t=1}^T  Q_0(s_0,\roota(\leaf_t))\right] &= \sum_{t=1}^T\mathbb{E}\left[ \mathbb{E}\left[ \left.Q_0(s_0,\roota(\leaf_t))\right|\hist_{t+1}\right]\right] \\
    &= \sum_{t=1}^T\mathbb{E}\left[ \mathbb{E}\left[\left.Y_t(\leaf_t)\right|\hist_{t+1}\right]\right]\\
    &= \mathbb{E}\left[\sum_{t=1}^T \left[Y_t(\leaf_t)\right]\right].
\end{split}
\end{equation*}
Let us define a policy $\tilde{\pi}_\tau$ that until time $\tau$ selects action $z_t$ according to $\pi$, and at time $\tau$ selects $z^*_\tau$. Let $\tilde{\hist}_{t+1}$ denote the history of following policy $\tilde{\pi}_\tau$ for $t$ steps. By definition, $\tilde{\hist}_{t+1} = \{ \hist_t, \leaf^*_t, r(\leaf^*_t)\}$
\begin{equation*}
\begin{split}
    \mathbb{E}\left[\sum_{t=1}^T \left[Q_0(s_0,A^*)\right]\right] &= \sum_{t=1}^T \mathbb{E}\left[ Q_0(s_0,A^*) \right] \\ 
    &= \sum_{t=1}^T \mathbb{E}\left[ \mathbb{E}\left[\left.Q_0(s_0,A^*) \right|\tilde{\hist}_{t+1}\right]\right] \\
    &= \sum_{t=1}^T \mathbb{E}\left[ \mathbb{E}\left[\left.Q_0(s_0,\roota(\leaf^*_t)) \right|\tilde{\hist}_{t+1}\right]\right] \\
    &= \sum_{t=1}^T\mathbb{E}\left[ \mathbb{E}\left[\left.Y_t(\leaf^*_t)\right|\tilde{\hist}_{t+1}\right]\right]\\
    &= \mathbb{E}\left[\sum_{t=1}^T \left[Y_t(\leaf^*_t)\right]\right].
\end{split}
\end{equation*}
\end{proof}

\subsection{Proof of Theorem \ref{thm:TSTS_regret}}

\begin{proof}
Our proof is broken into several parts, similarly to the analysis in \cite{russo2016information}. We begin with several information theoretic definitions. We then define the information ratio, and derive a general regret bound, and finally bound the information ratio in our case.

The Shannon entropy of a random variable $X$ is 
\begin{equation*}
    \ent(X) = \sum_{x} -P(X=x) \log P(X=x).
\end{equation*}
The mutual information $\infor(X;Y)$ between two random variables $X,Y$ satisfies
\begin{equation*}
    \infor(X;Y) = \ent(X) - \ent(X|Y) = \ent(Y) - \ent(Y|X) = \infor(Y;X).
\end{equation*}
The KL-divergence between two distributions is $D(P||Q) = \int \log\left(\frac{dP}{dQ}\right)dP$. It holds that (Fact 6 in \cite{russo2016information})
\begin{equation*}
    \infor(X;Y) = \sum_{x} P(X=x) D(P(Y|X=x)||P(Y)).
\end{equation*}

Let $P_t(X) = P(X|\hist_t)$, and $\mathbb{E}_t[\cdot] = \mathbb{E}[\cdot|\hist_t]$. Similarly, 
\begin{equation*}
    \ent_t(X) = \sum_{x} -P_t(X=x) \log P_t(X=x),
\end{equation*}
and $\infor_t(X;Y) = \ent_t(X) - \ent_t(X|Y)$.

As discussed in Section 3.1 of \cite{russo2016information}, we have that 
\begin{equation}\label{eq:exp_mutual_inf}
    \mathbb{E}\left[\infor_t(X;Y)\right] = \infor(X;Y|\hist_{t}).
\end{equation}

Let $Y_{t,a}$  denote the $a$'s component of $Y_t$. Define the information ratio,
\begin{equation}\label{eq:inf_ratio}
    \Gamma_t = \frac{\mathbb{E}_t\left[ Y_{t,\leaf^*_t} - Y_{t,\leaf_t}\right]^2}{\infor_t(\leaf^*_t; (\leaf_t, Y_{t,\leaf_t}))}.
\end{equation}

The following proposition bounds the regret using a bound on the information ratio.
\begin{proposition}\label{prop:regret_bound}
If $\Gamma_t \leq \bar{\Gamma}$ almost surely, we have that
\begin{equation*}
    \mathbb{E}\left[\textrm{Regret}(T)\right] \leq \sqrt{\bar{\Gamma}\ent(\leaf^*)T}.
\end{equation*}
\end{proposition}
\begin{proof}
We have that,
\begin{equation*}
    \begin{split}
        \mathbb{E}\left[\textrm{Regret}(T)\right] &= \mathbb{E} \left[\sum_{t=1}^T \left[Y_t(\leaf_t^*) - Y_t(\leaf_t)\right]\right] \\
        &= \mathbb{E} \left[\sum_{t=1}^T \mathbb{E}_t\left[Y_t(\leaf_t^*) - Y_t(\leaf_t)\right]\right] \\
        &= \mathbb{E} \left[\sum_{t=1}^T \sqrt{\Gamma_t \infor_t(\leaf_t^*; (\leaf_t, Y_{t,\leaf_t}))}\right] \\
        &\leq \sqrt{\bar{\Gamma}} \mathbb{E} \left[\sum_{t=1}^T \sqrt{\infor_t(\leaf_t^*; (\leaf_t, Y_{t,\leaf_t}))}\right] \\
        &\leq \sqrt{\bar{\Gamma}} \sqrt{T\mathbb{E} \left[\sum_{t=1}^T \infor_t(\leaf_t^*; (\leaf_t, Y_{t,\leaf_t}))\right]} \\
        &= \sqrt{\bar{\Gamma} T \mathbb{E} \left[\sum_{t=1}^T \infor_t(\leaf_t^*; (\leaf_t, Y_{t,\leaf_t}))\right]},
    \end{split}
\end{equation*}
where the first equality is by Proposition \ref{prop:reg_equiv}. The second equality is from the tower rule. The third equality is by definition \eqref{eq:inf_ratio}. The second inequality is by the Cauchy–Schwarz (CS) inequality, as follows. Define the linear inner product $<u,v> = \mathbb{E}[\sum_t u_t v_t]$. For $u = [1,\dots,1]$ and $v=[\sqrt{\infor_1},\dots, \sqrt{\infor_T}]$ we have $<u,v> = \mathbb{E}\left[\sum_{t=1}^T\sqrt{\infor_t}\right]$, $<u,u> = \mathbb{E}\left[\sum_{t=1}^T 1\right]=T$, and $<v,v> = \mathbb{E}\left[\sum_{t=1}^T \infor_t\right]$. Then, from CS, $\mathbb{E}\left[\sum_{t=1}^T\sqrt{\infor_t}\right] \leq \sqrt{T\mathbb{E}\left[\sum_{t=1}^T \infor_t\right]}$.

Next, define $Z_t = (O_t, \leaf_t, Y_{t,\leaf_t})$. We have that
\begin{equation*}
    \mathbb{E}\left[ \infor_t(\leaf_t^*; Z_t)\right] = \infor(\leaf_t^*; Z_t|Z_1,\dots,Z_{t-1}).
\end{equation*}
Therefore, 
\begin{equation*}
\begin{split}
    \mathbb{E}\left[ \sum_{t=1}^T \infor_t(\leaf_t^*; \leaf_t, Y_{t,\leaf_t}) \right] &\leq \mathbb{E}\left[ \sum_{t=1}^T \infor_t(\leaf_t^*; Z_t)\right]\\ &= \mathbb{E}\left[ \sum_{t=1}^T \infor_t(\leaf_t^*; O_t)\right]\\
    &\leq \mathbb{E}\left[ \sum_{t=1}^T \infor_t(\leaf^*; O_t)\right] \\
    &=\sum_{t=1}^T \infor(\leaf^*; O_t | O_1,\dots, O_{t-1}) \\
    &= \infor(\leaf^*; (O_1, \dots, O_T)) \\
    &= \ent(\leaf^*) - \ent(\leaf^* | O_1, \dots, O_T) \\
    &\leq \ent(\leaf^*),
\end{split}
\end{equation*}
where the first inequality is since $Z_t$ contains $Y_{t,A_t}$, the first equality is by the definition of $Y_t$, which, given the history, is independent of $\leaf_t^*$, and the third equality is from the chain rule of mutual information (Fact 5 in \cite{russo2016information}). The second inequality is by the data processing inequality, since $\leaf_t^*$ is a deterministic function of $\leaf^*$. Combining the results above gives the desired result. 
\end{proof}

We proceed, similarly to \cite{russo2016information}, to derive an equivalent form of the information ratio, which will facilitate further analysis.

\begin{proposition}\label{prop:gamma_equiv}
We have that
\begin{equation*}
    \infor_t(\leaf_t^*; (\leaf_t, Y_{t,\leaf_t})) = \sum_{a,a^*} P_t(\leaf_t^*=a^*)P_t(\leaf_t^*=a)\left[ D(P_t(Y_{t,a}|\leaf_t^*=a^*)||P_t(Y_{t,a}))\right],
\end{equation*}
and
\begin{equation*}
    \mathbb{E}_t\left[ Y_{t,\leaf_t^*} - Y_{t,\leaf_t}\right] = \sum_{a}P_t(\leaf_t^*=a)\left( \mathbb{E}_t\left[Y_t|\leaf_t^*=a\right] - \mathbb{E}_t\left[Y_t\right]\right).
\end{equation*}
\end{proposition}
\begin{proof}
We have
\begin{equation*}
    \begin{split}
        \infor_t(\leaf_t^*; (\leaf_t, Y_{t,\leaf_t})) &= \infor_t(\leaf_t^*; \leaf_t) + \infor_t(\leaf_t^*;Y_{t,\leaf_t}|\leaf_t) \\
        &= \infor_t(\leaf_t^*;Y_{t,\leaf_t}|\leaf_t) \\
        &= \sum_{a} P_t(\leaf_t=a) \infor_t(\leaf_t^*;Y_{t,\leaf_t}|\leaf_t=a) \\
        &= \sum_{a} P_t(\leaf_t=a) \infor_t(\leaf_t^*;Y_{t,a})) \\
        &= \sum_{a} P_t(\leaf_t=a) \sum_{a^*} P_t(\leaf_t^*=a^*) D(P_t(Y_{t,a}|\leaf_t^*=a^*))||P_t(Y_{t,a})) \\
        &= \sum_{a,a^*} P_t(\leaf_t=a)P_t(\leaf_t^*=a^*) D(P_t(Y_{t,a}|\leaf_t^*=a^*))||P_t(Y_{t,a})) \\
        &= \sum_{a,a^*} P_t(\leaf_t^*=a)P_t(\leaf_t^*=a^*) D(P_t(Y_{t,a}|\leaf_t^*=a^*))||P_t(Y_{t,a})),
    \end{split}
\end{equation*}
where the last equality uses the probability matching of TS: $P_t(\leaf_t=a) = P_t(\leaf_t^*=a)$.
Also,
\begin{equation*}
    \begin{split}
        \mathbb{E}_t\left[ Y_{t,\leaf_t^*} - Y_{t,\leaf_t}\right] &= \sum_{a} P_t(\leaf_t^*=a)\mathbb{E}_t\left[Y_{t,a}|\leaf_t^*=a\right] - \sum_{a} P_t(\leaf_t=a)\mathbb{E}_t\left[Y_{t,a}|\leaf_t=a\right] \\
        &= \sum_{a} P_t(\leaf_t^*=a)\left(\mathbb{E}_t\left[Y_{t,a}|\leaf_t^*=a\right] - \mathbb{E}_t\left[Y_{t,a}|\leaf_t=a\right]\right), \\
        &= \sum_{a} P_t(\leaf_t^*=a)\left(\mathbb{E}_t\left[Y_{t,a}|\leaf_t^*=a\right] - \mathbb{E}_t\left[Y_{t,a}\right]\right).
    \end{split}
\end{equation*}
where the second equality uses the probability matching of TS: $P_t(\leaf_t=a) = P_t(\leaf_t^*=a)$, and the third equality is since given the history, $Y_{t}$ is independent of $\leaf_t$. 
\end{proof}

We will use the following lemma.
\begin{lemma}\label{lem:pinsker}
Let $P,Q$ be two distributions of $X$ with support $[-B,B]$, such that $P$ is absolutely continuous with respect to $Q$. Then,
\begin{equation*}
    \mathbb{E}_P[X] - \mathbb{E}_Q[X] \leq B\sqrt{\frac{1}{2}D(P||Q)}.
\end{equation*}
\end{lemma}
\begin{proof}
We have 
\begin{equation*}
    \mathbb{E}_P[X] - \mathbb{E}_Q[X] = \sum_x x (P(x) - Q(x)) \leq \sum_x |x| |P(x) - Q(x)| \leq B \max_x |P(x) - Q(x)| \leq B \sqrt{\frac{1}{2}D(P||Q)},
\end{equation*}
where the last inequality is Pinsker's inequality.
\end{proof}

We are finally ready to bound the information ratio.
\begin{proposition}\label{prop:gamma_bound}
We have that $\Gamma_t \leq \frac{|\leaves|R_{max}^2 \hor^2}{2}$ almost surely.
\end{proposition}
\begin{proof}
We have
\begin{equation*}
    \begin{split}
        \mathbb{E}_t\left[ Y_{t,\leaf_t^*} - Y_{t,\leaf_t}\right]^2 &= \left( \sum_{a}P_t(\leaf_t^*=a)\left( \mathbb{E}_t\left[Y_{t,a}|\leaf_t^*=a\right] - \mathbb{E}_t\left[Y_{t,a}\right]\right) \right)^2 \\
        &\leq |\leaves| \sum_a P_t(\leaf_t^*=a)^2 \left( \mathbb{E}_t\left[Y_{t,a}|\leaf_t^*=a\right] - \mathbb{E}_t\left[Y_{t,a}\right]\right)^2 \\
        &\leq |\leaves| \sum_{a,a^*} P_t(\leaf_t^*=a)P_t(\leaf_t^*=a^*)\left( \mathbb{E}_t\left[Y_{t,a}|\leaf_t^*=a\right] - \mathbb{E}_t\left[Y_{t,a}\right]\right)^2 \\
        &\leq \frac{|\leaves| R_{max}^2 \hor^2}{2} \sum_{a,a^*} P_t(\leaf_t^*=a)P_t(\leaf_t^*=a^*) D\left(P_t(Y_{t,a}|\leaf_t^*=a)|| P_t(Y_{t,a})\right) \\
        &= \frac{|\leaves| R_{max}^2 \hor^2}{2} \infor_t(\leaf_t^*; (\leaf_t, Y_{t,\leaf_t})),
    \end{split}
\end{equation*}
where the first equality is by Proposition \ref{prop:gamma_equiv}. The first inequality is by CS, as follow. Consider the inner product $<u,v> = \sum_{a}u(a)v(a)$, and let $u = [1,\dots,1]$, and $v(a)=P_t(\leaf_t^*=a)\mathbb{E}_t\left[Y_t|\leaf_t^*=a\right] - \mathbb{E}_t\left[Y_t\right]$. Then by CS, $(\sum_{a}u(a)v(a))^2\leq \sum_{a}u(a)^2 \sum_{a'}v(a')^2 = |\leaves|\sum_a \left( P_t(\leaf_t^*=a)\mathbb{E}_t\left[Y_t|\leaf_t^*=a\right] - \mathbb{E}_t\left[Y_t\right] \right)^2$. The second inequality follows from the following fact: let $C(i,j)>0$. Then $\sum_{i,j} C(i)C(j)D(i)^2 = \sum_{i, j=i} C(i)C(j)D(i)^2 + \sum_{i, j \neq i} C(i)C(j)D(i)^2 \geq \sum_{i} C(i)^2 D(i)^2$. The third inequality is by Lemma \ref{lem:pinsker}, using that $|Y_{t,a}|\leq R_{max} \hor$, since $|Q(s_0,a)| \leq R_{max} \hor$. The last equality is again from Proposition \ref{prop:gamma_equiv}.
\end{proof}
Plugging Proposition \ref{prop:gamma_bound} in the bound of Proposition \ref{prop:regret_bound} completes the proof.
\end{proof}

\subsection{Proof of Proposition \ref{prop:TS_algorithm}}
We note that for any leaf $\leaf = (\tilde{s},\tilde{a})$, there is a unique branch leading to it, which we shall denote $b(\leaf) = \left\{ (s_0,a_0), (s_1,a_1),\dots , (s_k,a_k), (\tilde{s},\tilde{a}) \right\}$. We shall denote by $b^*_t = b(\leaf_t^*)$ the optimal branch. 

In the remainder of this proof, all probabilities are conditioned on $\hist_t$.
To simplify the notation, we omit this dependence.

From the sequential structure of the branch, we have that 
\begin{equation*}
P(b(\leaf) = b^*_t) = P((s_0,a_0) \in b^*_t)P((s_1,a_1) \in b^*_t|(s_0,a_0) \in b^*_t)\cdots P((\tilde{s},\tilde{a}) \in b^*_t|(s_k,a_k) \in b^*_t).
\end{equation*}
To see this, note that $P((s_k,a_k) \in b^*_t|(s_{k-1},a_{k-1}) \in b^*_t, (s_{k-2},a_{k-2}) \in b^*_t) = P((s_k,a_k) \in b^*_t|(s_{k-1},a_{k-1}) \in b^*_t)$, since if $(s_{k-1},a_{k-1})$ belongs to the optimal branch, its predecessor $(s_{k-2},a_{k-2})$ must also be on the optimal branch.

Observe that if $(s_{k-1},a_{k-1})$ is on the optimal branch, then the successor state $s_k$ must also be on the optimal branch. Therefore,
\begin{equation*}
    P((s_k,a_k) \in b^*_t|(s_{k-1},a_{k-1}) \in b^*_t) = P(Q(s_k,a_k) \in \argmax_a Q(s_k,a) ).
\end{equation*}
We therefore have that if $P(Q(s,a))$ in Algorithm \ref{alg:tsts} corresponds to the true posterior for each $s,a$, then the forward sampling procedure samples a branch from $P(b(\leaf) = b^*_t)$, and equivalently, samples a leaf from $P(\leaf_t^*)$. 

We now show by induction that $P(Q(s,a))$ in Algorithm \ref{alg:tsts} corresponds to the true posterior, which we shall denote here $P_{true}(Q(s,a))$ .
For any leaf $(s,a)$, by the independence assumption, $P(Q(s,a))$ is independent of other leaves or nodes in the tree, therefore after each update of the algorithm we have $P(Q(s,a)) = P_{true}(Q(s,a))$. Assume that for some node $s'$ and all actions $a'$, we have that $P(Q(s',a')) = P_{true}(Q(s',a'))$. Let $s,a$ be the state-action leading to $s'$. By definition, $P_{true}(Q(s,a))$ depends only on the decedents of $s,a$ in the tree. We therefore have that 
\begin{equation*}
    P_{true}\left(Q(s,a)\right) = P\left(r(s,a) + \max_{a'} \left\{ Q(s',a') \right\}\right) = P\left(Q(s,a)\right).
\end{equation*}